\newif\ifarxiv 
\arxivtrue

\newif\ifreview 
\reviewfalse

\newif\iffinal 
\finalfalse

\ifarxiv
\documentclass{article}
\else
\ifreview
\documentclass[dvipsnames,format=sigconf,anonymous=true,review=true]{acmart}
\else
\documentclass[dvipsnames,format=sigconf]{acmart}
\fi
\fi

\usepackage[utf8]{inputenc}

\ifarxiv
\usepackage[a4paper,left=2cm,right=2cm,top=2cm,bottom=2cm]{geometry}
\usepackage[onehalfspacing]{setspace}
\usepackage[numbers,longnamesfirst]{natbib}
\usepackage{amsthm,amsthm,amssymb}
\fi

\usepackage{amsmath,mathtools} 
\usepackage{dsfont}

\usepackage{booktabs} 
\usepackage{xspace}
\usepackage{graphicx}
\usepackage{xspace}
\usepackage{url}
\usepackage{enumerate}

\usepackage{caption} 
\usepackage{subcaption}

\usepackage[algo2e,ruled,vlined,linesnumbered]{algorithm2e} 
\SetAlgoSkip{}

\usepackage{balance} 

\allowdisplaybreaks[3] 

\usepackage{tikz}
\usetikzlibrary{shapes,arrows}
\usetikzlibrary{decorations}
\usetikzlibrary{plotmarks}
\usetikzlibrary{calc}
\usetikzlibrary{mindmap}
\usetikzlibrary{shadows}
\usetikzlibrary{backgrounds}
\usetikzlibrary{patterns}
\usetikzlibrary{shapes.symbols}

\usepackage{pgfplots}
\usepackage{pgfplotstable}
\usepgfplotslibrary{statistics}

\ifarxiv
\newtheorem{theorem}             {Theorem}
\newtheorem{lemma}      [theorem]{Lemma}
\newtheorem{corollary}  [theorem]{Corollary}

\newtheorem{definition} [theorem]{Definition}

\fi

\usepackage{tabularx,colortbl}
\newcommand{\myc}{\cellcolor{orange!30}}

\makeatletter
\g@addto@macro\bfseries{\boldmath}
\makeatother

\newcommand{\ie}{i.\,e.\xspace}

\newcommand{\eg}{e.\,g.\xspace}

\newcommand{\Z}{\mathbb{Z}}
\newcommand{\N}{\mathbb{N}}
\newcommand{\R}{\mathbb{R}}

\DeclareMathOperator{\Unif}{Unif}                         
\DeclareMathOperator{\Bin}{Bin}                           
\DeclareMathOperator{\Nor}{\ensuremath{\mathcal{N}}}      

\newcommand{\prob}[1]{\Pr\left(#1\right)}                 
\newcommand{\Prob}[1]{\Pr\left(#1\right)}                 
\newcommand{\expect}[1]{\mathrm{E}\left[#1\right]}        
\newcommand{\E}[1]{\mathrm{E}\left[#1\right]}        

\newcommand{\vecone}{\vec{1}}

\newcommand{\ptscd}[2]{$(#1,#2)$\xspace}                  
\newcommand{\ptscdsup}[2]{$(#1,#2)$-superior\xspace}      
\newcommand{\ptscdsep}[2]{$(#1,#2)$-separable\xspace}     
\newcommand{\numsuppts}{\ensuremath{S}\xspace}            

\newcommand{\TZ}{\textsc{TZ}\xspace}                      
\newcommand{\LO}{\textsc{LO}\xspace}                      

\newcommand{\ones}[1]{|#1|_1}                             
\newcommand{\zeroes}[1]{|#1|_0}                           

\newcommand{\OMM}{\textsc{OMM}\xspace}                             
\newcommand{\OMMfull}{\textsc{OneMinMax}\xspace}                   

\newcommand{\LOTZ}{\textsc{LOTZ}\xspace}                           
\newcommand{\LOTZfull}{\textsc{LeadingOnesTrailingZeroes}\xspace}  

\newcommand{\cdist}{\ensuremath{\textsc{cDist}}\xspace}   

\newcommand{\nsga}{NSGA\nobreakdash-II\xspace}

\ifarxiv\else 
\fi

\newcommand{\toggleplot}[1]{{\textcolor{red}{Plots removed to increase compilation speed. Use command $\backslash$toggleplot in preamble to reinsert them.}}}

\iffinal
\newcommand{\andre}[1]{\textcolor{green!50!black}{}}
\newcommand{\dirk}[1]{\textcolor{purple}{}}
\newcommand{\cuong}[1]{\textcolor{orange}{}}
\newcommand{\bahare}[1]{\textcolor{magenta}{}}

\newcommand*{\todo}[1]{\textcolor{red}{}}
\else
\newcommand{\andre}[1]{\textcolor{green!50!black}{[(Andre) #1]}}
\newcommand{\dirk}[1]{\textcolor{purple}{[(Dirk) #1]}}
\newcommand{\cuong}[1]{\textcolor{orange}{[(Cuong) #1]}}
\newcommand{\bahare}[1]{\textcolor{magenta}{[(Bahare) #1]}}

\newcommand*{\todo}[1]{\textcolor{red}{\textrm{(TODO: #1)}}}
\fi

\ifarxiv
\else
\copyrightyear{2023}
\acmYear{2023}
\setcopyright{acmlicensed}\acmConference[GECCO '23]{Genetic and Evolutionary Computation Conference}{July 15--19, 2023}{Lisbon, Portugal}
\acmBooktitle{Genetic and Evolutionary Computation Conference (GECCO '23), July 15--19, 2023, Lisbon, Portugal}
\acmPrice{15.00}
\acmDOI{10.1145/3583131.3590421}
\acmISBN{979-8-4007-0119-1/23/07}
\fi

\ifarxiv
\else
\fi

\begin{document}

\ifarxiv
\author{
    Duc-Cuong~Dang
    \and Andre~Opris
    \and Bahare~Salehi
    \and Dirk~Sudholt}
\date{}
\else
\author{Duc-Cuong~Dang}
\affiliation{
    \institution{University of Passau\city{Passau}\country{Germany}}
}
\author{Andre~Opris}
\affiliation{
    \institution{University of Passau\city{Passau}\country{Germany}}
}
\author{Bahare~Salehi}
\affiliation{
    \institution{University of Passau\city{Passau}\country{Germany}}
    \institution{Shiraz University\city{Shiraz}\country{Iran}}
}
\author{Dirk~Sudholt}
\affiliation{
    \institution{University of Passau\city{Passau}\country{Germany}}
}
\fi

\title{Analysing the Robustness of NSGA-II under Noise}

\ifarxiv
\maketitle
\fi

\begin{abstract}
Runtime analysis
has produced many results on the efficiency of
simple
evolutionary algorithms
like the (1+1)~EA, and
its analogue
called GSEMO
in evolutionary multiobjective optimisation (EMO).
Recently, the first runtime analyses of the famous and highly cited
EMO algorithm \nsga
have emerged,
demonstrating
that practical algorithms with thousands of applications can be rigorously analysed.
However, these results only show that \nsga has the same performance guarantees as GSEMO and it is
unclear
how and when \nsga can outperform GSEMO.

We study this question in noisy optimisation and consider
a 
noise model
that adds
large amounts of posterior noise to all objectives with some
constant
probability~$p$ per evaluation.
We show that GSEMO fails badly on every noisy fitness function as it tends to remove large parts of the population
indiscriminately.
In contrast, \nsga is able to handle the noise efficiently
on \textsc{LeadingOnesTrailingZeroes}
when $p<1/2$,
as
the algorithm
is able to preserve
useful search points even in the presence of noise. We identify a phase transition at $p=1/2$ where the expected time to cover the Pareto front changes from polynomial to exponential.
To our knowledge, this is the first proof that \nsga can outperform GSEMO and the first runtime analysis of \nsga in noisy optimisation.
\end{abstract}

\ifarxiv\else
\fi

\ifarxiv\else 
\maketitle
\fi

\section{Introduction}\label{sec:intro}

Decision making is ubiquitous in everyday life, and often can be formalised as
an optimisation problem. In many situations, one may want to examine the
trade-off between compromises before making a decision. Sometimes
these compromises cannot be accurately evaluated, \ie due to a lack of available
information when a decision has to be made.
These two critical but practical settings correspond to the areas
    Multi-Objective Optimisation (MOO)
    and Optimisation under Uncertainty (OUU), respectively,
which have been studied in both
    Economics,
    Operational Research,
    and Computer Science 
%
\cite{Ben2009,Birge2011,Forman2001,Goh2007,Hughes2001,Llora2003,Ravi1993,Teich2001}.

MMO is an area where evolutionary multi-objective (EMO) algorithms have
shown to be among the most efficient optimisation techniques \cite{Tan2005}.
Particularly, the Non-dominated Sorting Genetic Algorithm (\nsga) 
is a highly influential framework to build algorithms for MMO, and the original
paper \cite{Deb2002} is one of the most highly cited papers in evolutionary
computation and beyond.

Recently, 
\nsga
was analysed by rigorous
mathematical means using \emph{runtime analysis}.
%
In a nutshell, runtime analysis studies the performance guarantees
and drawbacks of randomised search heuristics, like Evolutionary Algorithms (EAs),
from a Computer Science perspective \cite{Jansen2013}.
%
The basic approach is
to bound the expectation of the random running time
$T$ (number of iterations or function evaluations)
of a given algorithm on a problem until
a global optimum is found in case of a single objective. 
Extending this to MMO,
$T$ is the time to find and cover the whole \emph{Pareto front}.
The algorithm is said to be \emph{efficient} if this expectation is polynomial in the
problem size, and it
is \emph{inefficient} if the expectation is exponential.

%

Runtime analyses have led to a better understanding of the capabilities and limitations of EAs, for example concerning the advantages of population diversity~\cite{Dang2017}, the benefits of using crossover~\cite{Doerr2015,Sudholt2016,Corus2018a,Lengler2020}, or the robustness of populations in stochastic optimisation~\cite{Dang2015,GiessenK16,LehreQ21,QianBYTY21}.
It can give advice on how to set algorithmic parameters; it was used to identify phase transitions between efficient and inefficient running times for parameters of common selection operators~\cite{Lehre2010a}, the offspring population size in comma strategies~\cite{Rowe2013} or the mutation rate for difficult monotone pseudo-Boolean functions~\cite{Doerr2012c,LenglerS18}.
Runtime analysis has also inspired novel designs for EAs with practical impacts, \eg
    choosing mutation rates from a heavy-tailed distribution to escape from local optima~\cite{Doerr2017-fastGA},
    parent selection in steady-state EAs~\cite{CorusLOW21}, 
    selection in non-elitist populations with power-law ranking~\cite{DangELQ22}, or
    choosing mutation rates adaptively throughout the run~\cite{Doerr2019opl,LehreQ22,QinL22}.

Runtime analyses in MOO started out with the simple SEMO algorithm and 
the global SEMO (GSEMO)~\cite{Laumanns2004,Giel2010}. Both algorithms keep non-dominated solutions in the population. If a new offspring $x$ is created that is not dominated by the current population, it is added to the population and all search points that are weakly dominated by~$x$ are removed. Despite its simplicity, it was shown to be effective in AI applications, \eg \cite{Qian_Bian_Feng_2020},
where it is called PO(R)SS.

The first theoretical runtime analysis of \nsga (without crossover) was performed by~\citet{ZhengLuiDoerrAAAI22}. They showed that \nsga covers the whole Pareto front for test functions \LOTZ and \OMM (see Section~\ref{sec:prelim}) in expected $O(\mu n^2)$ and $O(\mu n \log n)$ function evaluations, respectively, where $\mu$ is the population size and $n$ is the problem size (number of bits). These results require a population of size $\mu \ge 4(n+1)$, hence the best upper bounds are $O(n^3)$ and $O(n^2 \log n)$, respectively, that also apply to
(G)SEMO~\cite{Laumanns2004,Giel2010}.

This breakthrough result spawned several very recent papers and already led to several new insights and proposals for improved algorithm designs.
\Citet{Bian2022PPSN} proposed a new parent selection mechanism called \emph{stochastic tournament selection} and showed that \nsga equipped with this operator covers the Pareto front of \LOTZ in expected time $O(n^2)$.
\Citet{Zheng2022} proposed to re-compute the crowding distance during the selection process and proved (using \LOTZ as a test case) that this improved the spread of individuals on the Pareto front.
\Citet{DoerrQu22} proposed to use heavy-tailed mutations in \nsga and quantified the speedup on multimodal test problems. \Citet{DoerrQu2022} and \citet{Dang2023} independently demonstrated the advantages of crossover in \nsga. In terms of limitations, \Citet{ZhengArXiv2022} investigated the inefficiency of \nsga for more than two objectives and \citet{DoerrQu2022a} gave lower bounds for the running time of \nsga.




Despite these rapidly emerging research works, one important research question remains open. So far all comparisons of \nsga with GSEMO show that \nsga has the same performance guarantees as GSEMO. Even though \nsga is a much more complex algorithm, we do not have an example where \nsga was proven to outperform the simple GSEMO algorithm; thus we have not yet unveiled the full potential of \nsga.

Here we provide such an example from noisy optimisation.


\textbf{Our contribution}:
We show that \nsga can drastically outperform GSEMO on a noisy \textsc{LeadingOnesTrailingZeroes} test function.
%
To this end, we introduce a deliberately simple posterior noise model called the $(\delta, p)$-Bernoulli
noise model, in which a fixed noise $\delta\in \R$ is added to the fitness
in all objectives and in each evaluation with some noise probability $p$.
When $\delta$ is positive and sufficiently large, for maximisation problems every noisy solution always dominates every noise-free solution.
%
In this setting, we prove in Theorem~\ref{the:negative-result-GSEMO}
that it is difficult for GSEMO to grow its population hence
the algorithm is highly
inefficient 
under noise on arbitrary noisy fitness functions.
%
In contrast, for the noise model with a constant $p<1/2$, we show in
Theorem~\ref{thm:nsga-ii-noisy-lotz} that \nsga is efficient
on the noisy \LOTZfull function, if its population size
is sufficiently large. This result can be easily extended to other functions.
The reason for
this performance gap
is that \nsga
keeps dominated solutions in its population while GSEMO immediately removes them.
We also prove in Theorem~\ref{thm:nsga-ii-noisy-lotz-p05} that the behaviour of
\nsga without crossover dramatically changes for the noise probability slightly
above $1/2$, \ie it suddenly becomes inefficient. 
Our theoretical results are complemented with empirical results on both the
Bernoulli noise model and an additive Gaussian noise, which confirm
the advantageous robustness of \nsga over GSEMO.
As far as we know, this is the first proof that \nsga can outperform GSEMO,
and the first runtime analysis of \nsga under uncertainty.

\section{Preliminaries}\label{sec:prelim}

By $\log(\cdot)$ we denote the logarithm of base~2.
$\R$, $\Z$ and $\N$ are the sets of real, integer and natural numbers respectively.
For $n \in \N$, define $[n] := \{1,\dots,n\}$ and
$[n]_0 \coloneqq [n] \cup \{0\}$.
%
We use $\vecone$ to denote the all-ones vector $\vecone:=(1,\dots,1)$.
For a bit string $x:=(x_1,\dots,x_n)\in\{0,1\}^n$, we use $\ones{x}$ to denote
its number of $1$-bits, \ie $\ones{x}=\sum_{i=1}^{n}x_i$,
and similarly $\zeroes{x}$ to denotes its number of zeroes,
\ie $\zeroes{x}=\sum_{i=1}^{n}(1-x_i)=n-\ones{x}$.
We use standard asymptotic notation with symbols $O, \Omega, o$~\cite{Cormen2009}.

This paper focuses on the multi-objective optimisation in
a discrete setting, specifically the maximisation of a $d$-objective
function 
    $f(x):=(f_1(x),\dots,f_d(x))$
    where $f_i\colon \{0,1\}^n \rightarrow \Z$ for each $i \in [d]$.
We define
    $f_{\min}:=\min\{f_i(x) \mid i \in [d], x\in\{0,1\}^n\}$,
    and $f_{\max}:=\max\{f_i(x) \mid i \in [d], x \in \{0,1\}^n\}$.

\begin{definition}
Consider a $d$-objective function $f$:
\begin{itemize}
\item For $x, y \in \{0, 1\}^n$, we say $x$ \emph{weakly dominates} $y$
      written as $x \succeq y$ (or $y \preceq x$) if $f_i(x) \geq f_i(y)$ for all $i\in[d]$;
      $x$ \emph{dominates} $y$ written as $x \succ y$ (or $y \prec x$)
      if one inequality is strict.
\item A set of points which covers all possible fitness values not
      dominated by any other points in $f$ is called \emph{Pareto front}.
      A single point from the Pareto front is called \emph{Pareto optimal}.
\end{itemize}
\end{definition}

The weakly-dominance and dominance relations are \emph{transitive},
\eg $x \succ y \wedge y \succ z$ implies $x \succ z$.
%
When $d=2$, the function is referred as \emph{bi-objective}.
Two basic bi-objective functions studied in theory of evolutionary computation
are \LOTZfull and \OMMfull, which can be shortly written as \LOTZ and \OMM respectively. In $\LOTZ(x) \coloneqq (\LO(x), \TZ(x))$ we count the number $\LO(x)$ of leading ones in $x$ (the length of the longest prefix containing only ones) and the number $\TZ(x)$ of trailing zeros in~$x$ (the length of the longest suffix containing only zeros). \OMM minimises and maximises the number of ones:
$\OMM(x)
    := (\ones{x}, \zeroes{x})$.

\begin{algorithm2e}[ht]
	Initialize $P_0 \sim \Unif( (\{0,1\}^n)^{\mu})$\\
	Partition $P_0$ into layers $F^1_0,F^2_0,\dots$ of non-dominated fitnesses, then for each layer $F^i_0$ compute the crowding distance $\cdist(x,F^i_0)$ for each $x \in F^i_0$\\
	\For{$t:= 0 \to \infty$}{
		Initialize $Q_t:=\emptyset$\\
		\For{$i:=1 \to \mu/2$}{
			Sample $p_1$ and $p_2$, each by a binary tournament \label{alg:nsga-ii:tournament}\\
			Sample $u \sim \Unif([0,1])$\\
			\If{$u<p_c$}
			{Create $s_1, s_2$ by
                crossover on $p_1, p_2$}
			\Else{Create $s_1, s_2$ as exact copies of $p_1, p_2$}
			Create $s'_1$ by bitwise mutation on $s_1$ with rate $1/n$\\
			Create $s'_2$ by bitwise mutation on $s_2$ with rate $1/n$\\
			Update $Q_t:=Q_t \cup \{s'_1,s'_2\}$\\
		}
		Set $R_t := P_t \cup Q_t$\\
		Partition $R_t$ into layers $F^1_{t+1},F^2_{t+1},\dots$ of non-dominated fitnesses, then for each layer $F^i_{t+1}$ compute $\cdist(x,F^i_{t+1})$ for each $x \in F^i_{t+1}$\\
		Sort $R_t$ lexicographically 
		by $(1/i, \cdist(x,F^i_{t+1}))$\label{alg:nsga-ii:survival-sort}\\
		Create the next population $P_{t+1} := (R[1],\dots,R[\mu])$\\
	}
	\caption{NSGA-II Algorithm \cite{Deb2002}}
	\label{alg:nsga-ii}
\end{algorithm2e}

\nsga \cite{Deb2002,NSGAIICode2011}
is summarised in Algorithm~\ref{alg:nsga-ii} for bitwise mutation.
In each generation, a population $Q_t$ of $\mu$
new offspring search points
are
created
through binary tournament,
crossover and mutation.
The binary tournament
in line~\ref{alg:nsga-ii:tournament} uses
the same criteria
as the sorting procedure in line~\ref{alg:nsga-ii:survival-sort}
which will be detailed below.
The
crossover is only applied with some probability $p_c \in (0,1)$
to produce two solutions $s_1, s_2$. 
Otherwise $s_1,s_2$ are the exact copies of the winners of the tournaments.
The \emph{bitwise mutation} on $s_1,s_2$  creates two offspring
by flipping each bit of the input independently with probability $1/n$.
%
Our positive result for \nsga (Theorem~\ref{thm:nsga-ii-noisy-lotz}) holds for arbitrary crossover operators as it only relies on steps without crossover.
To simplify the analysis, we assume that the tournaments for parent selection are performed independently and with replacement.

During the survival selection,
the parent and offspring populations $P_t$ and $Q_t$ are joined into $R_t$,
and then partitioned into
layers $F^1_{t+1},F^2_{t+1},\dots$ by the \emph{non-dominated sorting algorithm} \cite{Deb2002}.
    The layer $F^1_{t+1}$ consists of all non-dominated points,
    and $F^i_{t+1}$ for $i>1$ only contains points that are dominated by
    those from $F^1_{t+1},\dots,F^{i-1}_{t+1}$.
In each layer, the \emph{crowding distance} is computed for each search point,
then the points of $R_t$ are sorted with respect to the indices
of the layer that they belong to
as the primary criterion, and then with the computed crowding distances
as the secondary criterion.
Only the $\mu$ best solutions of $R_t$
form the next population.

Let $M:=(x_1,x_2,\dots,x_{|M|})$ be a
multi-set
of search points.
The crowding distance $\cdist(x_i,M)$ of $x_i$ with respect to $M$
is computed as follows. At first sort $M$ as $M=(x_{k_1},\dots,x_{k_{\vert{M}\vert}})$
with respect to
each objective $k \in [d]$
separately.
Then
\begin{align}
	\cdist(x_i, M)
	&:= \sum_{k=1}^{d} \cdist_{k}(x_i, M), \label{eq:crowd-dist}
	\text{ where }\\
	\cdist_{k}(x_{k_i}, M)
	&\!:= \! \begin{cases}
		\infty\; & \text{if } i \in \{1, |M|\},\\
		\frac{f_k\left(x_{k_{i-1}}\right) - f_k\left(x_{k_{i+1}}\right)}{f_k\left(x_{k_1}\right) - f_k\left(x_{k_{M}}\right)} & \text{otherwise.}
	\end{cases}\!\!\!\!\label{eq:crowd-dist-eachdim}
\end{align}
The first and last ranked individuals are
always
assigned an infinite crowding distance. The remaining individuals
are then assigned the differences between the values of $f_k$ of
those ranked directly above and below the search point and normalized
by the difference between $f_k$ of the first and last ranked.

The GSEMO algorithm is
shown in Algorithm~\ref{alg:gsemo}.
Starting from one randomly generated solution, in each generation
a new search point $s'$ is created by
    crossover, with some probability $p_c \in (0,1)$,
    and bitwise mutation with parameter $1/n$ afterwards
where parents are selected uniformly at random.
If $s'$ is not dominated by any solutions of the current population $P_t$
then it is added to the population, and those weakly dominated by $s'$
are removed from the population. 
The population size
$|P_t|$ is unrestricted for GSEMO.

\begin{algorithm2e}[ht]
	Initialize $P_0:=\{s\}$ where $s \sim \Unif(\{0,1\}^n)$\\
	\For{$t:= 0 \to \infty$}{
		Sample $p_1 \sim \Unif(P_t)$ \\
		Sample $u \sim \Unif([0,1])$ \\
		\If{$u<p_c$}{
			Sample $p_2 \sim \Unif(P_t)$\\
			Create $s$ by
            crossover between $p_1$ and $p_2$}
		\Else{
			Create $s$ as a copy of $p_1$}
		Create $s'$ by bitwise mutation on $s$ with rate $1/n$\\
		\If{$s'$ is \textbf{not} dominated by any individual in $P_t$}{
			Create the next population $P_{t+1} := P_t \cup \{s\}$ \\
			Remove all $x \in P_{t+1}$ weakly dominated by $s'$\\
		}
	}
	\caption{GSEMO Algorithm}
	\label{alg:gsemo}
\end{algorithm2e}

Note that GSEMO and \nsga
are \emph{invariant} under
a translation of the objective function, that is, they behave identically
on $f$ and on $f+\vec{c}$ where $\vec{c}$ is a fixed vector.

%

\section{The Posterior Bernoulli Noise Model}\label{sec:bernoulli-noise-model}

Since our aim is to demonstrate that \nsga is more robust to noise than GSEMO, we choose the simplest possible noise model under which the desired effects are evident. Our noise model is inspired by concurrent work~\cite{Lengler2023}. Noise can either be present or absent, the strength of the noise is fixed, and noise is applied to all objectives uniformly. Using a simple noise model facilitates a theoretical analysis and simplifies the presentation. We will discuss possible extensions to more realistic noise models, and we will consider one further noise model (posterior Gaussian noise) in our empirical evaluation (Section~\ref{sec:experiment}).

In our posterior noise model, instead of optimising the real fitness~$f$, the algorithm only has access to a noisy fitness function, denoted as $\tilde{f}$ that may return fitness values obscured by noise.

(This is different to \emph{prior noise}~\cite{Dang2014,Dang2016,Droste2004}, in which the search point is altered before the fitness evaluation.)
%
In our noise model the fitness is altered
    by a fixed additive term $\delta \in \R$ in all objectives,
    with some probability $p>0$.
We refer to $|\delta|$ as the \emph{noise strength}, and $p$ as the
\emph{noise probability} or \emph{frequency}.

\begin{definition}
Given
    a noise strength $\delta \in \R$ 
    and a noise probability $p \in [0, 1]$,
the noisy optimisation of a $d$-objective fitness function
$f$ under the $(\delta,p)$-Bernoulli noise model has $\tilde{f}$ defined
as
\[
    \tilde{f}(x) := \begin{cases}
        f(x) + \delta \cdot \vecone & \text{with probability } p,\\
        f(x) & \text{otherwise.}
    \end{cases}
\]
\end{definition}
When
$\tilde{f}(x) = f(x) + \delta \cdot \vecone$ we call $x$
    a \emph{noisy} search point and otherwise we call it \emph{noise-free}.
Note that the \emph{expected} fitness vector of any search point~$x$ is
\[
    \expect{\tilde{f}(x)} = f(x) + p\delta \cdot \vecone
\]
and hence
optimising $f$ is equivalent to optimising the expectation of $\tilde{f}$;
in other words, this is equivalent to \emph{stochastic optimisation} \cite{Birge2011}.
%
When $p \in \{0, 1\}$ the noisy function $\tilde{f}$ is deterministic
and equal to~$f$ apart from a possible translation by~$\delta$ in all objectives,
thus
\nsga and GSEMO will behave the same as operating on $f$.

%
%
Since we aim to study the robustness of
these original algorithms, we refrain from noise reduction techniques like re-sampling~\cite{Akimoto2015,QianBYTY21,Bian00T21}.

%

We assume that noise is drawn independently for all search points
in a generation.
This reflects a setting where noise is
generated from an external source, \eg disturbances when evaluating the fitness.
We assume however that in each generation the noisy fitness values of evaluated individuals are stored temporarily for that generation. So, if the fitness of an individual is queried multiple times in the same generation, the noisy fitness value from the first evaluation in that generation is returned.

Now we obtain a specific noise model by setting $\delta > f_{\max} - f_{\min}$.
In this case, noise boosts the fitness of a search point in an extreme way;
its fitness immediately strictly dominates that of every noise-free search point.
%
For all $\delta > f_{\max} - f_{\min}$
\nsga  (or GSEMO) behaves identically on the noise model $(\delta, p)$ as on $(-\delta, 1-p)$,
because in the latter model
the roles of noisy and noise-free search points are swapped
and the fitness is translated by $-\delta\cdot\vec{1}$.
The latter
model corresponds to a setting where noise may destroy the fitness of a search
point. This scenario
is closely related to practice when optimising problems
with constraints that are typically met, but where noise may violate constraints
and this incurs a large penalty.

%
%

\section{GSEMO Struggles With Noise}\label{sec:analysis-gsemo}

We show that noise is hugely detrimental for SEMO and GSEMO. Since both algorithms reject all search points that are weakly dominated by a new offspring, if the offspring falsely appears to dominate good quality search points, the latter are being lost straight away. The following analysis shows that, for sufficiently large noise strengths~$\delta$, there is a good chance that creating a noisy offspring will remove a fraction of the population, irrespective of the fitness of population members. This makes it impossible to grow the population to a size necessary to cover the Pareto front of a function.

\begin{theorem}
\label{the:negative-result-GSEMO}
Consider GSEMO on an arbitrary fitness function $f$ with noise strength $\delta>f_{\max}-f_{\min}$ and noise probability $0 < p < 1$.
For any functions $t(n), \alpha(n) \in \N$, starting with an arbitrary initial population of size at most $\lceil p \alpha(n) \rceil$, the probability of the population reaching a size of at least $\alpha(n)$ in the first $t(n)$ generations is at most $t(n) \cdot (1-p/2)^{\lfloor (1-p)\alpha(n) \rfloor-1}$ and the expected number of generations is at least $(1-p/2)^{-\lfloor (1-p)\alpha(n) \rfloor+1}$.
\end{theorem}
\begin{proof}
Let $\mu_t \le \alpha(n)-1$ denote the population size at time~$t$. We call a step~$t+1$  \emph{shrinking} if $\mu_{t+1} \le \lceil p\alpha(n)\rceil +1$. Since GSEMO only adds at most one search point to the population, the condition $\mu_t \le \lceil p \alpha(n) \rceil$ implies a shrinking step since $\mu_{t+1} \le \mu_t + 1 \le \lceil p \alpha(n) \rceil+1$. Hence we assume $\mu_t \ge \lceil p \alpha(n) \rceil + 1$ in the following. A sufficient condition for a shrinking step is
to create a noisy offspring and to evaluate at most $\lceil p\alpha(n)\rceil$ parents as noisy. Since the noisy offspring dominates all noise-free search points and GSEMO removes these from the population, only noisy parents may survive. The probability of the offspring being noisy is~$p$.
Conditional on this event, each of the $\mu_t$ search points in the population survives with probability~$p$, independently from one another. Then the number of survivors is given by a binomial distribution $\mathrm{Bin}(\mu_t, p)$. We have
\begin{align*}
    & \prob{\mathrm{Bin}(\mu_t, p) \le \lceil p\alpha(n) \rceil}
    \ge \prob{\mathrm{Bin}(\mu_t, p) \le \lceil p\mu_t \rceil} \ge 1/2
\end{align*}
since the median of the binomial distribution is at most $\lceil p\mu_t \rceil$.

Thus, a shrinking step occurs with probability at least $p/2$. If $\mu_t \le \lceil p\alpha(n)\rceil+1$, the population can only grow to $\alpha(n)$ if there is a sequence of $\alpha(n) - (\lceil p\alpha(n)\rceil+1) = \alpha(n) + \lfloor -p \alpha(n) \rfloor -1= \lfloor (1-p)\alpha(n) \rfloor -1$ steps that are all not shrinking. The probability of such a sequence is at most $(1-p/2)^{\lfloor (1-p)\alpha(n) \rfloor -1}$. If a shrinking step occurs, the population size drops to at most $\lceil p\alpha(n)\rceil+1$ and we can re-iterate the argument. Taking a union bound over the first $t(n)$ steps proves the first claim. Noticing that each attempt to reach a population size of at least $\alpha(n)$ requires at least one evaluation, the expected number of evaluations is bounded by the expectation of a geometric random variable with parameter $(1-p/2)^{\lfloor (1-p)\alpha(n) \rfloor-1}$, which is $(1-p/2)^{-\lfloor (1-p)\alpha(n) \rfloor+1}$.
\end{proof}
Processes where the current state shows a multiplicative expected decrease, plus some additive term, were recently analysed in~\cite{DoerrNegativeMultiplicativeDrift}. The expected change of the state was described as \emph{negative multiplicative drift with an additive disturbance}. Lower bounds are given on the expected time to reach some target value~$M$. However, these bounds are linear in~$M$, whereas the bounds from Theorem~\ref{the:negative-result-GSEMO} are exponential in the target $\alpha(n)$.

Theorem~\ref{the:negative-result-GSEMO} shows that, if $\alpha(n)$ is chosen as the size of the smallest Pareto set, it takes GSEMO exponential expected time in $\alpha(n)$ to reach a population that covers the whole Pareto front.
\begin{theorem}
Consider GSEMO on an arbitrary fitness function $f$ for which every Pareto set has size at least $\alpha(n)$.
Then, for every constant~$p \in (0, 1)$, in the $(f_{\max} - f_{\min}+1, p)$-Bernoulli noise model, the expected time for GSEMO to
cover the whole Pareto front is $2^{\Omega(\alpha(n))}$.
\end{theorem}

Since all Pareto sets of well-known multiobjective test functions have size at least~$n+1$, we get the following.
\begin{corollary}\label{cor:negative-result-GSEMO}
%
For every constant $p \in (0, 1)$, in the $(n+1, p)$-Bernoulli noise model the expected time for GSEMO on
\LOTZfull or \OMMfull
to  cover the whole Pareto set is $2^{\Omega(n)}$.
\end{corollary}


We are confident that the arguments in this section extend to other posterior noise models, e.\,g.\ adding Gaussian posterior noise. If the noise strength is not too small, there is a good chance that the offspring might be sampled with large positive noise, and then population members with a negative noise contribution may be dominated and be removed as argued above.

Note, however, that to achieve domination, the offspring must be at least as good as a population member in all objectives. In our Bernoulli noise model, we add the same noise of~$\delta$ to all objectives. If noise is determined independently for each objective and a value of~$\delta$ is added with probability~$p$, the offspring is guaranteed to dominate every noise-free population member if noise is applied in all dimensions. For $d$ dimensions, the probability of this event is $p^d$. If $d$ and $p$ are both constant, this is only a constant-factor difference, and thus if adding noise uniformly to all objectives yields an exponential lower bound from Theorem~\ref{the:negative-result-GSEMO}, the same holds when adding noise independently for each objective.

%

\section{NSGA-II is Robust to Noise if $p < \frac{1}{2}$}\label{sec:analysis-nsga-ii}


%

For \nsga with the Bernoulli noise model, when
$\delta$ is sufficiently large,
noisy search points do not interfere with
the dynamics of non-dominated layers containing noise-free points,
\ie in the calculation of the crowding distances. This is captured by
the following concepts, which are illustrated in Figure~\ref{fig:cd-sep-multiset}.

\begin{definition}\label{def:cd-sep-multiset}
Let $C \in \Z, D \in \N_0$ be some integers,
and let
$f(x):=(f_1(x),f_2(x))$ be a
discrete
bi-objective function.

%
\begin{itemize}
    \item A point $x\in \{0,1\}^n$ is called a \emph{\ptscd{C}{D}-point} if
    $f_1(x)=C+\ell \wedge f_2(x)=C+m$ for some $\ell, m \in [D]_0$.
    \item A point $x\in \{0,1\}^n$ is \emph{\ptscdsup{C}{D}} if it dominates all
    \ptscd{C}{D}-points, \ie $f_1(x)> C+D \wedge f_2(x)>C+D$.
    \item A multi-set $P$ of points of $f$ is called \emph{\ptscdsep{C}{D}} if
    it only contains \ptscd{C}{D}-points and \ptscdsup{C}{D} points.
    %
\end{itemize}
\end{definition}

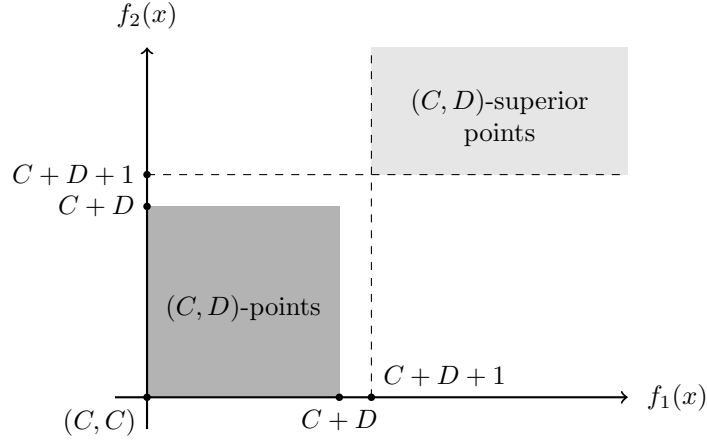
\begin{figure}[h]\centering
\begin{tikzpicture}[x=1.2em,y=1.2em]
    \tikzstyle{vertex}=[draw,circle,thick,minimum size=1em,scale=0.2,fill=black,align=center];
    \tikzstyle{axis}=[draw,black,thick,->];
    \tikzstyle{arrow}=[draw,black,thick,-stealth];
    \tikzstyle{edge}=[draw,black,thick];

    \fill[color=black!30!white] (0,0) rectangle (6,6);
    \fill[color=black!10!white] (7,7) rectangle (15,11);

    \draw[axis] (0,-1)  -> (0,11);
    \node[label=90:{$f_2(x)$}] at (0,11) {};
    \draw[axis] (-1,0)  -> (15,0);
    \node[label=0:{$f_1(x)$}] at (15,0) {};

    \node[vertex,label=-120:{$(C,C)$}] at (0,0) {};
    \node[vertex,label=-90:{$C+D$}]  at (6,0) {};
    \node[vertex,label=30:{$C+D+1$}]  at (7,0) {};
    \draw[dashed] (0,7) -- (15,7);
    \node[vertex,label=-180:{$C+D$}]  at (0,6) {};
    \node[vertex,label=-180:{$C+D+1$}] at (0,7) {};
    \draw[dashed] (7,0) -- (7,11);

    \node[label={\ptscd{C}{D}-points}] at (3,1.8) {};    
    \node[label={[align=center]\ptscdsup{C}{D}\\points}] at (11,7.3) {}; 
\end{tikzpicture}
\caption{Illustration of \ptscdsep{C}{D} multi-sets. Only the two shaded areas contain search points.}
\ifarxiv\else\Description{Illustration of \ptscdsep{C}{D} multi-sets.}\fi
\label{fig:cd-sep-multiset}
\end{figure}

To show that \nsga can optimise a function and cover a Pareto front, one has to prove that
the progress made so far by the optimisation process is maintained and
that Pareto optimal solutions are not being lost in future
generations \cite{Dang2023,DoerrQu22,ZhengLuiDoerrAAAI22}. Such arguments were first used by~\cite{ZhengLuiDoerrAAAI22} and later on in~\cite{Bian2022PPSN,DoerrQu22,Dang2023}. In
\cite{Dang2023} these arguments were extracted and summarised in a lemma~\cite[Lemma~7]{Dang2023}.
We adapt the lemma to our case 
as follows.

\begin{lemma}\label{lem:nsga-ii-protect-layer}
Consider two
consecutive
generations $t$ and $t+1$
of the \nsga
maximising a bi-objective function $f(x):=(f_1(x),f_2(x))$
where $f_i(x)\colon \{0,1\}^n \rightarrow \Z$ for each $i\in\{1,2\}$,
and two numbers $C \in \Z, D \in \N_0$
such that $R_t$ (\ie the joint parent and offspring population)
is \ptscdsep{C}{D}.
Then we have:
\begin{itemize}
    \item[(i)] Any layer $F_t^i$
    composed of
    only \ptscd{C}{D}-points has at most
    ${4(D+1)}$ individuals with positive crowding distances.
    The same result holds for layers $F_{t+1}^i$
    that only have
    \ptscd{C}{D}-points.
    \item[(ii)] If $R_t$ has at most \numsuppts \ptscdsup{C}{D} points for
    some $\numsuppts\in \N_0$ and the population size $\mu$ satisfies
    $\mu \geq 4(D+1) + \numsuppts$, then the following result holds.
    If there is a \ptscd{C}{D}-point $x \in P_t$, then there must exist
    a \ptscd{C}{D}-point $y \in P_{t+1}$ with either $f(y)=f(x)$
    or $y \succ x$.
\end{itemize}
\end{lemma}
\begin{proof}
The layers of $R_t$ (the union of parents and offspring) are separated between those
containing \ptscd{C}{D}-points and those having the superior ones.
The layers of \ptscdsup{C}{D} points are higher ranked (have a lower index) than \ptscd{C}{D}-points.

    (i)
    It suffices to prove the result for $F_{t+1}^i$
    with only \ptscd{C}{D}-points of $R_t$ then it also holds for the
    same type of layers in $P_t, P_{t+1}$ because these populations
    are sub-multi-sets of $R_t$.
    The remaining proof arguments for (i) are identical to those in
    the proof of result (i) of Lemma~7 in~\cite{Dang2023},
    with their $F_{t}^1$ being replaced by our $F_{t+1}^i$.
    Thus we omit it here and refer to~\cite[Lemma~7]{Dang2023} for details.
    %
    %

    We note one insight from the proof for later use.
    The proof shows that for each fitness vector $(a,b)$ of the layer
    there is at least a point with a positive crowding distance.

    (ii)
    Let $i^*$ be the smallest integer such that the layer
    $F^{i^*}_{t+1}$ of $R_t$ contains only \ptscd{C}{D}-points,
    \ie the layers $F^{j}_{t+1}$ with $j< i^*$ only contain \ptscdsup{C}{D}-points.
    The condition on $R_t$ means that $\sum_{j<i^*} |F_{t+1}^j|\leq \numsuppts$,
    thus it follows from (i) and $\mu\geq 4(D+1) + \numsuppts$ that
    $P_{t+1}$ will contain all search points from $F_{t+1}^{i^*}$ with
    positive crowding distance, in addition to the \ptscdsup{C}{D} points of $R_t$.

    We have the following cases for the \ptscd{C}{D}-point $x$:

    \underline{Case 1}: If none of the \ptscd{C}{D}-points of $R_t$ dominates $x$,
    then clearly $x \in F_{t+1}^{i^*}$.
    As remarked at the end of the proof of (i), there must exist
    a $y \in F_{t+1}^{i^*}$ with a positive crowding distance and with $f(y)=f(x)$.
    Thus $y$ will be kept in $P_{t+1}$.

    \underline{Case 2}: If some of the \ptscd{C}{D}-points of $R_t$ dominate $x$,
    then let $y$ be such a point. We may assume that $y$ is not dominated
    by any other point of $R_t$ because if there is a $y' \in R_t$ dominating $y$, we can choose $y'$ instead of~$y$ and iterate this argument until a non-dominated point is found.
    This implies that
    $y \in F_{t+1}^{i^*}$ and as in the previous case, there exists
    a $y' \in F_{t+1}^{i^*}$ with a positive crowding distance and with $f(y')=f(y)$.
    Thus $y'$ will be kept in $P_{t+1}$ and we have $y'\succ x$.
\end{proof}

Now we use Lemma~\ref{lem:nsga-ii-protect-layer} to show that \nsga can find the Pareto front of \LOTZ efficiently when the noise probability is at most a constant less than $1/2$.
%
%
Roughly speaking, the result shows that with a sufficiently large population,
a sub-population of \nsga can still evolve its noise-free search points,
thus noise has a minimal effect on the optimisation process.
For example, $\mu \ge 9(n+1)$ meets the condition for all noisy LOTZ
functions with $p \le 1/4 \wedge \delta>n$ or $p \ge 3/4 \wedge \delta<-n$.
We will use this setting later in our experiments.

\begin{theorem}\label{thm:nsga-ii-noisy-lotz}
Consider the $(\delta, p)$-model with noise strength $\delta > n$
    and constant noise probability
        $p\in \left(0,\frac{1}{2(1+c)}\right)$
    for some constant~${c > 0}$.
Then \nsga with population size $\mu\geq \frac{4(n+1)}{1-2(1+c)p}$
and $p_c \le 1 - 2^{-o(n)}$ finds and covers the whole Pareto front of noisy \LOTZ in
    $\mathcal{O}\left(n^2/(1-p_c)\right)$ expected generations and
    $\mathcal{O}\left(\mu n^2/(1-p_c)\right)$ expected fitness evaluations.
The
result also holds for the $(-\delta, 1-p)$-model using the
same conditions. 
\end{theorem}
\begin{proof}

For $\delta>n$, the noise model
guarantees that all noisy solutions dominate all noise-free solutions,
thus the populations $P_t, Q_t$ and $R_t$ are typically \ptscdsep{0}{n}
with the superior points being the noisy ones.
%
Lemma~\ref{lem:nsga-ii-protect-layer}~(i) with $C=0, D=n$ then implies
that there are no more than $4(n+1)$
individuals
with positive crowding distances in every layer $F_t^i$, or $F_{t+1}^i$ of noise-free individuals.

Furthermore, in each generation of the algorithm, the expected number of parents
in $P_t$ that are noise-free after re-evaluation is $(1-p)\mu$, thus by an additive
Chernoff bound~\cite[Theorem~1.10.7]{doerr-theory-chapter}, the probability of having at least $(1-(1+c)p)\mu$ noise-free parents
and conversely at most $(1+c)p\mu$ noisy ones is
at least
$1-e^{-2c^2p^2\mu}=
1 - e^{-\Omega(n)}$. Similarly, during the evaluation of the offspring $Q_t$,
with probability $1-e^{-\Omega(n)}$, there are at least $(1-(1+c)p)\mu$
noise-free solutions and at most $(1+c)p\mu$ noisy ones. If either one
of these two conditions does not occur during a generation, we refer to this as
a \emph{bad event}. The probability of a bad event is at most $2e^{-\Omega(n)}
=e^{-\Omega(n)}$.
%
Given the rarity of bad events, we apply the typical run method
with the restart argument, see Chapter~5.6 of \cite{Jansen2013}.
%
We therefore divide the run of the algorithm into
two phases, each phase is associated with a goal.
The phases last for at most $T_1$ and
$T_2$ generations respectively,
and have the failure probability
of at most $p_1$ and $p_2$ respectively.
A failure means that a bad event happens during the random length of a phase.
As that the following analysis works for any initial population, such
a failure is no worse than restarting the analysis on the resulting population.
Under this assumption, the expected number of generations until the Pareto
front is covered is at most $(\E{T_1}+\E{T_2})/(1-p_1-p_2)$.


\underline{Phase~1}: Create a first Pareto optimal solution.

Define
$i:=\max\{\LO(x)+\TZ(x) \mid x \in P_t\}$ and note that all Pareto optimal solutions
$x$ have $\LO(x) + \TZ(x) = n$. Thus, the phase is completed once $i=n$. 
%
We now
consider a point $y\in P_t$ so that
$\LO(y)+\TZ(y)=i$ and $y$ has positive crowding distance,
to give a lower bound
for the probability of increasing~$i$
in a generation.

During
a binary
tournament, the probability of selecting $y$ as the first competitor
is $1/\mu$.
%
The probability of the other competitor being a noise-free solution with zero crowding
distance is bounded from below as follows. There are at least $(1-(1+c)p)\mu$ noise-free parents in $P_t$, and each noise-free layer has at most $4(n+1)$ individuals with positive crowding distances.
Thus, the sought probability is at least
$(1-(1+c)p)\left(1-\frac{4(n+1)}{\mu}\right)
    \geq (1-(1+c)p)\left(1-\frac{4}{4/(1-2(1+c)p)}\right)
    = \Omega(1)$.
So, even when
$y$
is noise-free
the probability of it winning the tournament is
at least $2\cdot \Omega(1)\cdot (1/\mu) = \Omega(1/\mu)$ where the factor $2$
accounts for
the exchangeable roles of the competitors.

To create an offspring $z$ that dominates $y$ 
with $\LO(z)+\TZ(z) \geq i + 1$ 
it suffices to select $y$ as a parent, to skip crossover and to
flip one specific bit of $y$ during mutation, while keeping the other bits
unchanged (this mutation has probability $1/n \cdot (1-1/n)^{n-1} \ge 1/(en)$).
These events occur with probability
$
 s_i
    := (1-p_c)(1/en) \cdot \Omega(1/\mu)
     = \Omega((1-p_c)/(\mu n))$. During $\mu/2$ offspring productions, the probability of creating such a solution
$z$ is
$1 - (1-s_i)^{\mu/2}
    \geq \frac{s_i\mu/2}{s_i \mu/2 + 1}
    = \frac{s_i \mu}{s_i \mu + 2}$,
    where the inequality follows from~\cite[Lemma~6]{Badkobeh2015}.

During survival selection, since we have at most $(1+c)p\mu$ noisy solutions in $P_t$ and $Q_t$, respectively,
there are at most $2(1+c)p\mu$ noisy
solutions in $R_t = P_t \cup Q_t$.
As
the population size $\mu$
satisfies $\mu \geq 2(1+c)p\mu + 4(n+1)$, 
Lemma~\ref{lem:nsga-ii-protect-layer}~(ii)
    with $C=0, D=n, \numsuppts=2(1+c)p\mu$
first implies that
even when no such $z$ individual is created
an individual
with the same fitness as
$y$ always survives
to the next generation
$P_{t+1}$;
in other words,
$i$
cannot decrease.
Second, when one individual 
$z$ is created, regardless of whether it is evaluated with noise or not,
$z$, or an individual weakly dominating it, survives.

So the expected number of generations of this phase
is at most:
\begin{align*}
\E{T_1} \le \sum_{i=0}^{n-1} \left(1 + \frac{2}{s_i\mu}\right) 
    = n + \frac{2}{\mu}\sum_{i=0}^{n-1} \mathcal{O}\left(\frac{\mu n}{1-p_c}\right)
    = \mathcal{O}\left(\frac{n^2}{1-p_c}\right).
\end{align*}
The failure probability of the phase
is bounded from
above by the law of total probability and union bounds as
$p_1 \leq  \sum_{t=1}^{\infty} \prob{T_1=t} \cdot t e^{-\Omega(n)}
    = e^{-\Omega(n)} \E{T_1}
    = o(1)$
since $1 - p_c=2^{-o(n)}$.

\underline{Phase~2}: Cover the whole Pareto front.

Now that the population $P_t$ contains Pareto-optimal individuals,
while the whole Pareto front has not been covered, the population contains a Pareto-optimal
individual $y$ with positive crowding distance such that one of the fitness vectors
$(\LO(y)-1, \TZ(y)+1)$ or $(\LO(y)+1, \TZ(y)-1)$ is not yet present in~$f(P_t)$.

As argued for Phase~1, the probability that
during one offspring production, the sequence of operations selection, crossover,
and mutation produces an offspring $z$ with a missing fitness vector
is at least
$
 s_i' \coloneqq \Omega((1-p_c)/(\mu n))$.
Again with $\mu/2$ trials, the probability of creating $z$ per generation
is
$1 - (1-s_i')^{\mu/2}\geq s_i'\mu/(s_i'\mu + 2)$.

During the survival selection, we again use
Lemma~\ref{lem:nsga-ii-protect-layer}~(ii)
with the same parameters
to argue that
Pareto optimal fitness vectors will never be removed entirely from the population.
%
There can be at most $n$ missing Pareto optimal fitness vectors to cover, thus the
expected number of generations to complete this phase is at most
\begin{align*}
\E{T_2} 
    \leq \sum_{i=1}^{n} \left(1 + \frac{2}{s_i'\mu}\right)
    = \sum_{i=1}^{n} \left(1 + \frac{2}{\mu}\cdot \mathcal{O}\left(\frac{\mu n}{1-p_c}\right)\right)
    = \mathcal{O}\left(\frac{n^2}{1-p_c}\right).
\end{align*}
and by a similar argument as in the other phase,
we also have $p_2 = o(1)$ given $1-p_c=2^{-o(n)}$.

The total expected number of generations until the Pareto front is covered is
bounded from above by
$(\E{T_1} + \E{T_2}) / (1 - p_1 - p_2)
    = {\mathcal{O}\left(n^2/(1-p_c)\right)/(1-o(1))}
    = \mathcal{O}\left(n^2/(1-p_c)\right)
$.
%
The bound on the expected
number of evaluations follows from the fact that the number of solutions evaluated
in each generation is $2\mu = O(\mu)$.
\end{proof}

Our analysis can be easily extended to show similar results for other functions.
For instance, the expected time to cover the Pareto front of noisy \OMM in the
same noise model is at most $\mathcal{O}\left(\mu n\log{n}/(1-p_c)\right)$.
\ifarxiv
\begin{theorem}\label{thm:nsga-ii-noisy-omm}
Consider the $(\delta, p)$-model with noise strength $\delta > n$ and constant noise probability $p\in \left(0,\frac{1}{2(1+c)}\right)$ for some constant~${c > 0}$.
Then \nsga with population size $\mu\geq \frac{4(n+1)}{1-2(1+c)p}$
and $p_c \le 1 - 2^{-o(n)}$ covers the whole Pareto front of
the noisy \OMM function in
    $\mathcal{O}\left(\frac{n\log{n}}{1-p_c}\right)$ expected generations and
    $\mathcal{O}\left(\frac{\mu n\log{n}}{1-p_c}\right)$ expected fitness evaluations.
The
result also holds for the $(-\delta, 1-p)$-model using the
same conditions. 
\end{theorem}
\begin{proof}
We follow the same approach as in the proof of Theorem~\ref{thm:nsga-ii-noisy-lotz},
by
using the typical run method with the restarting argument, and define the
bad events the same way. That is, a bad event occurs in generation $t$
    either if more than $(1+c)p\mu$ parents in $P_t$ are re-evaluated with noise,
    or if more than $(1+c)p\mu$ offspring individuals in $Q_t$ are evaluated with noise,
thus the probability of such an event is at most $2e^{-\Omega(n)}=e^{-\Omega(n)}$.
Since any search point is Pareto optimal for \OMM,
we only have a single phase of covering the Pareto front.
%
Like \LOTZ, the noisy \OMM function is also \ptscdsep{0}{n}, thus
the conditions to apply Lemma~\ref{lem:nsga-ii-protect-layer}~(i) are all fulfilled
given the same setting for $\mu$.

If the Pareto front is not covered entirely, then there must exist search points
$z \notin P_t$ next to a search point $y\in P_t$ with a positive crowding distance,
\ie $|\ones{z}-\ones{y}|=1 \wedge |\zeroes{z}-\zeroes{y}|=1$.
(These events are equivalent for 0-1 strings of equal length.)
Let $\ones{y}=i$ thus $\zeroes{y}=n-i$, and we will focus
on the case
    $\ones{z} = \ones{y}+1 = i+1 \wedge \zeroes{z} = \zeroes{y}-1 = n - i - 1$
as the reasoning is symmetric
in the other case of
    $\ones{z} = \ones{y}-1 \wedge \zeroes{z} = \zeroes{y}+1$.
Similar to the argument for \LOTZ,
with probability $\Omega(1/\mu)$, $y$ is selected as a parent and with probability
$1-p_c$ it creates an offspring by mutation only. To create such a point $z$ by mutation,
it suffices to flip one of the $n - i$ $0$-bits of $y$ to $1$ while keeping the
rest of the bits unchanged, and this happens with probability
$\frac{n-i}{n}\left(1-\frac{1}{n}\right)^{n-1}
    =\Omega({\frac{n-i}{n}})$.
So, the probability that a search point $z$ is created during one offspring production
is $s_i :=\Omega(\frac{(1-p_c)(n-i)}{\mu n})$. With $\mu/2$ offspring productions,
the chance of creating a solution $z$ is
$1 - (1-s_i)^{\mu/2}
    \geq \frac{s_i\mu/2}{s_i \mu/2 + 1}
    = \frac{s_i \mu}{s_i \mu + 2}$.

Once a point $z$ is created, then
by Lemma~\ref{lem:nsga-ii-protect-layer}~(ii), similarly to the case of \LOTZ,
a search point with
    fitness $f(z)$
is always kept in the population.
Thus, starting from $y$, the expected number of generations to cover
fitness vectors 
$(i+1,n-i-1),(i+2,n-i-2),\dots,(n,0)$, if they do not yet exist in the population,
is at most
\[
\sum_{k=i}^{n-1} \left(1 + \frac{2}{\mu s_i}\right)
    \leq \sum_{k=0}^{n-1} \left(1 + \mathcal{O}\left(\frac{2n}{(1-p_c)(n-i)}\right)\right)
    = \mathcal{O}\left(\frac{n\log{n}}{1-p_c}\right).
\]
By symmetry, the same bound holds to cover the fitness vectors 
$(i-1,n-i+1),(i-2,n-i+2),\dots,(0,n)$.
So, the expected number of generations to cover the whole front is no more than
    $\E{T} = \mathcal{O}\left(\frac{n\log{n}}{1-p_c}\right)$,
and the failure probability of the phase is at most
    $\sum_{t=1}^{\infty} \prob{T = t} \cdot t e^{-\Omega(n)}
     =e^{-\Omega(n)}\E{T}
     =o(1)$ given $1 - p_c = 2^{-o(n)}$.
%
Thus, the expected runtime of the algorithm is at most
    $\mathcal{O}\left(\frac{n\log{n}}{1-p_c}\right) \cdot \frac{1}{1-o(1)}
     = \mathcal{O}\left(\frac{n\log{n}}{1-p_c}\right)$
generations,
or, equivalently, at most $\mathcal{O}\left(\frac{\mu n\log{n}}{1-p_c}\right)$
fitness evaluations since only $2\mu$ evaluations are required per generation.
\end{proof}
\else
This result is omitted due to the space limit.
\fi
%

\section{Phase Transition for NSGA-II at $p=\frac{1}{2}$}

We have seen that, when $\delta>n$, with an appropriate scaling of the population size
\nsga
can handle any constant noise probability $p$ approaching $1/2$
from below.
Our result from Theorem~\ref{thm:nsga-ii-noisy-lotz} does not cover the case $p>1/2$
because 
when approaching $1/2$ from above,
the progress of optimisation has to rely on having a sufficient number of good individuals
that are evaluated with noise. The dynamic of the algorithm is therefore different,
in fact,
%
%
the next theorem shows that a noise probability around $p>1/2$ leads to poor
results on \LOTZ. For the sake of simplicity, we omit crossover and leave
an analysis including crossover for future work.
%

\begin{theorem}\label{thm:nsga-ii-noisy-lotz-p05}
Consider the \nsga with population size $\mu \in [n+1,\infty) \cap O(n)$
and crossover turned off ($p_c=0$) on the noisy \LOTZ function with the
$(\delta,p)$-noise model. If $\delta>n$ and $p$ is a constant such that
$1/2 < p < 10/19$ then \nsga requires $e^{\Omega(n)}$ generations with overwhelming probability to cover the whole Pareto front.
\end{theorem}

The analysis will show that the number of Pareto-optimal individuals is bounded with overwhelming probability. We first give a bound on the probability of creating a Pareto-optimal individual.
\begin{lemma}
\label{lem:reaching-F-via-mutation}
Let $F:=\{1^i 0^{n-i}\mid i \in [n]_0\}$ be the Pareto set of \LOTZ and consider a standard bit mutation creating $y$ from~$x$. Then
\[
    \Prob{y \in F} \le
    \begin{cases}
        1/e + 3/n & \text{ if $x \in F$}\\
        3/n & \text{ otherwise.}
    \end{cases}
\]
\end{lemma}
\begin{proof}
We assume $n \ge 3$ as otherwise the claimed probability bound is at least~$1$, which is trivial.
Starting from a parent $x=1^i 0^{n-i} \in F$, an offspring in $F$ is created if $x$ is cloned, or if it is mutated into another search point $1^j 0^{n-j}$ with $j \neq i$. We have $\Prob{y = x} = (1-1/n)^n \le 1/e$ and
$\Prob{y = 1^{j}0^{n-j}} \le n^{-|i-j|}$ for all $j \in [n]_0 \setminus \{i\}$ since (depending on whether $j < i$ or $j > i$) either the last $|i-j|$ 1-bits or the first $|i-j|$ 0-bits in $x$ must be flipped. The sum of all probabilities is at most
\[
    \frac{1}{e} + \sum_{d=1}^\infty 2n^{-d} = \frac{1}{e} + 2 \cdot \frac{1/n}{1-1/n} \le \frac{1}{e} + \frac{3}{n}
\]
where the last step used $1/(n-1) \le 3/(2n)$ for $n \ge 3$.

Now assume $x \notin F$. If the Hamming distance to the closest point in~$F$ is at least~2, at least two specific bits must be flipped to create a specific offspring $1^j0^{n-j} \in F$. This has probability at most $1/n^2$. Taking a union bound over $n+1$ possible values of~$j$ yields a probability bound of $(n+1)/n^2 \le 2/n$. If $x$ has Hamming distance~1 to some search point $1^i 0^{n-i} \in F$, it either has a single 0-bit among bits $\{1, \dots, i-1\}$ bits (and an all-zeros suffix) or a single 1-bit among bits $\{i+2, \dots, n\}$ bits (and an all-ones prefix). (Bit positions $i$ and $i+1$ are excluded as otherwise $x \in F$.)
In the former case, if the 0-bit is at position $i-1$, $x$ has Hamming distance~1 to both $1^i 0^{n-i}$ and $1^{i-2}0^{n-i+2}$ and Hamming distance at least~2 to all other search points in~$F$. If the 0-bit is at some smaller index, $x$ has Hamming distance at least~2 to all search points in $F \setminus \{1^i0^{n-i}\}$. The case of a single 1-bit is symmetric. Hence there are always at most two search points at Hamming distance~1 in~$F$, and each one is reached with probability at most $1/n$. To reach any other point in~$F$, two bits must be flipped. Taking a union bound over all probabilities yields a probability bound of $2/n + n \cdot 1/n^2 = 3/n$ in this case.
\end{proof}

Now we prove Theorem~\ref{thm:nsga-ii-noisy-lotz-p05}.
\begin{proof}[Proof of Theorem~\ref{thm:nsga-ii-noisy-lotz-p05}]
Recall that, owing to $\delta>n$, all noisy individuals dominate all noise-free
ones. The Pareto set of \LOTZ is
    $F:=\{1^i 0^{n-i}\mid i \in [n]_0\}$.
We use variables
    $X_t:=|P_t \cap F|$
to denote the number of individuals on $F$ in generation $t$. The Pareto front can only be found if $X_t \ge |F| = n+1$.
It is easy to see that the initial population at time~$0$ will have $X_t < n+1$ with probability at least $1-e^{-\Omega(n)}$, and we assume that this happens.
Since crossover is turned off, in each generation the $\mu$ offspring
are created independently by the binary tournament selection, followed by bitwise mutation.

We first find an upper bound $p_{\mathrm{tour}}$ on the probability that an
individual in $F$ is returned by an application of tournament selection.
A necessary event is
to sample at least one of the two competitors from~$F$. Each competitor is sampled from $F$ with probability $X_t/\mu$. By a union bound, the probability of at least one competitor being from~$F$ is at most $2X_t/\mu$. Thus, $p_{\mathrm{tour}} \le 2X_t/\mu$.

Starting from a parent $x=1^i 0^{n-i} \in F$, the probability of the offspring also being in $F$ is at most $1/e + 3/n$ by Lemma~\ref{lem:reaching-F-via-mutation}.
From a parent $x \notin F$, the probability of creating an offspring
on $F$ is at most $3/n$ by Lemma~\ref{lem:reaching-F-via-mutation}.
Together, the probability of a parent selection followed by mutation creating a search point in~$F$ is at most
\[
    p_{\mathrm{tour}}(1/e+3/n) + 3/n \le \frac{p_{\mathrm{tour}}}{e} + \frac{6}{n} \eqqcolon q.
\]
Let $Y_t \coloneqq |Q_t \cap F|$ be the number of Pareto optimal points in
the offspring population, then $Y_t$ can be bounded by a binomial distribution, $Y_t \preceq \Bin(\mu, q)$.
As $\mu = O(n)$, we have $\expect{Y_t} \leq 2 X_t/e + O(1)$.
%
%
We analyse the distribution of $X_{t+1}$ 
%
and consider two cases:

\underline{Case 1}: $n/4 \leq X_t \leq n$. For any constant $\delta \in (0,9 e/20-1)$
and a sufficiently large $n$, by a Chernoff bound it holds that
\begin{align*}
\prob{Y_t \geq 9X_t/10 \mid X_t \ge n/4}
    &\leq \prob{Y_t \geq (1+\delta)\left(2 X_t/e+O(1)\right)}\\
    &\leq e^{-2 \delta^2 X_t/(3e) - \delta^2 O(1)} = e^{-\Omega(n)}.
\end{align*}
Thus the probability of creating more than
$9X_t/10$ offspring
on $F$ is exponentially small.
%
Additionally, since $p$ is a constant below $10/19$ and above $1/2$
by two applications of Chernoff bounds we have the following.
The probability of having more than $(10/19)(X_t+Y_t)$ noisy individuals on $F$
is 
$e^{-\Omega(X_t+Y_t)}=e^{-\Omega(n)}$. 
The probability of having less than  $\mu$ noisy individuals among
the $2\mu$ individuals in $R_t$ is
 $e^{-\Omega(\mu)}=e^{-\Omega(n)}$. 
Since the survival selection only keeps the $\mu$
best among the $2\mu$ individuals, if the latter event does not occur
then none of the noise-free individuals will survive to the next generation.
Thus, when none of the three events occur,
\ie with a probability of at least $1-e^{-\Omega(n)}$ by a union bound,
$X_{t+1}\le(10/19)(X_t + Y_t)<(10/19)(X_t + 9X_t/10)=X_t$. In other words,
\begin{align*}
\prob{X_{t+1}\geq X_t}
    = e^{-\Omega(n)}.
\end{align*}

\underline{Case 2}: $0 \leq X_t < n/4$.
Since $\expect{Y_t} \le 2X_t/e + O(1) \le n/(2e) + O(1) \eqqcolon m$, we have
$\prob{Y_t \ge 2m} \le e^{-m/3} = e^{-\Omega(n)}$.
This implies $X_t + Y_t \le n/4 + n/e + O(1) \le n$ (for $n$ large enough) with probability $1-e^{-\Omega(n)}$.
If this happens then $X_{t+1} \le n$.

Combining these two cases gives that to reach $X_t \ge n+1$,
an event must occur that has probability $e^{-\Omega(n)}$. The expected time until such an event occurs is $e^{\Omega(n)}$, thus \nsga requires at least $e^{\Omega(n)}$ generations in expectation.
\end{proof}

\section{Experiments}\label{sec:experiment}

To complement the theoretical results, experiments were conducted to
compare the robustness of \nsga with GSEMO on \LOTZfull and \OMMfull.
We considered two noise models, the first is the $(\delta, p)$-Bernoulli noise model using
$\delta \coloneqq n+1$ and various noise probabilities
$p \in \{2^{-2},2^{-3}, \dots, 2^{-6}\}
           \cup \{0.4, 0.5,0.6\}
           \cup \{1 - 2^{-5},1- 2^{-4},1- 2^{-3},1-2^{-2}\}$
to cover noise probabilities
close to $0, 1/2$, and $1$ respectively.
%
%
%
In the second model, to investigate in how far our results translate to
more general noise models, we consider posterior Gaussian noise as in~\cite{FriedrichKKS17}. The noisy
fitness of a search point~$x$ is defined as follows, $\Nor(0, \sigma^2)$ denoting the normal
distribution with mean~$0$ and standard deviation~$\sigma$:
\[
\tilde{f}(x)
    :=f(x) + \vecone\cdot \delta
\text{ where } \delta \sim \Nor(0, \sigma^2).
\]
A Gaussian noise is always added to the fitness after evaluation (to all objectives), that is, there is no noise probability in this model.
For $\sigma = n$ there is a constant probability of $\tilde{f}(x) \succ \tilde{f}(y)$ for any two search points~$x, y$, irrespective of their true fitness. Hence we vary the standard deviation as $\sigma:=n\cdot q$ where $q \in \{2^0, 2^{-1}, 2^{-3},2^{-4}\}$.
We used problem size $n \in \{20, 30, 40\}$,
$p_c = 0.9$ and one-point crossover. For \nsga, the population size
is set to $\mu=9(n+1)$.
%
%
For each experiment, 50 runs were performed.
In each run, the algorithm is stopped either
when the whole Pareto front is covered and then the number of iterations
is recorded, or when the number of fitness evaluations exceeded
$10 n^3$.


\begin{figure}[th]\centering
\includegraphics[width=0.8\linewidth]{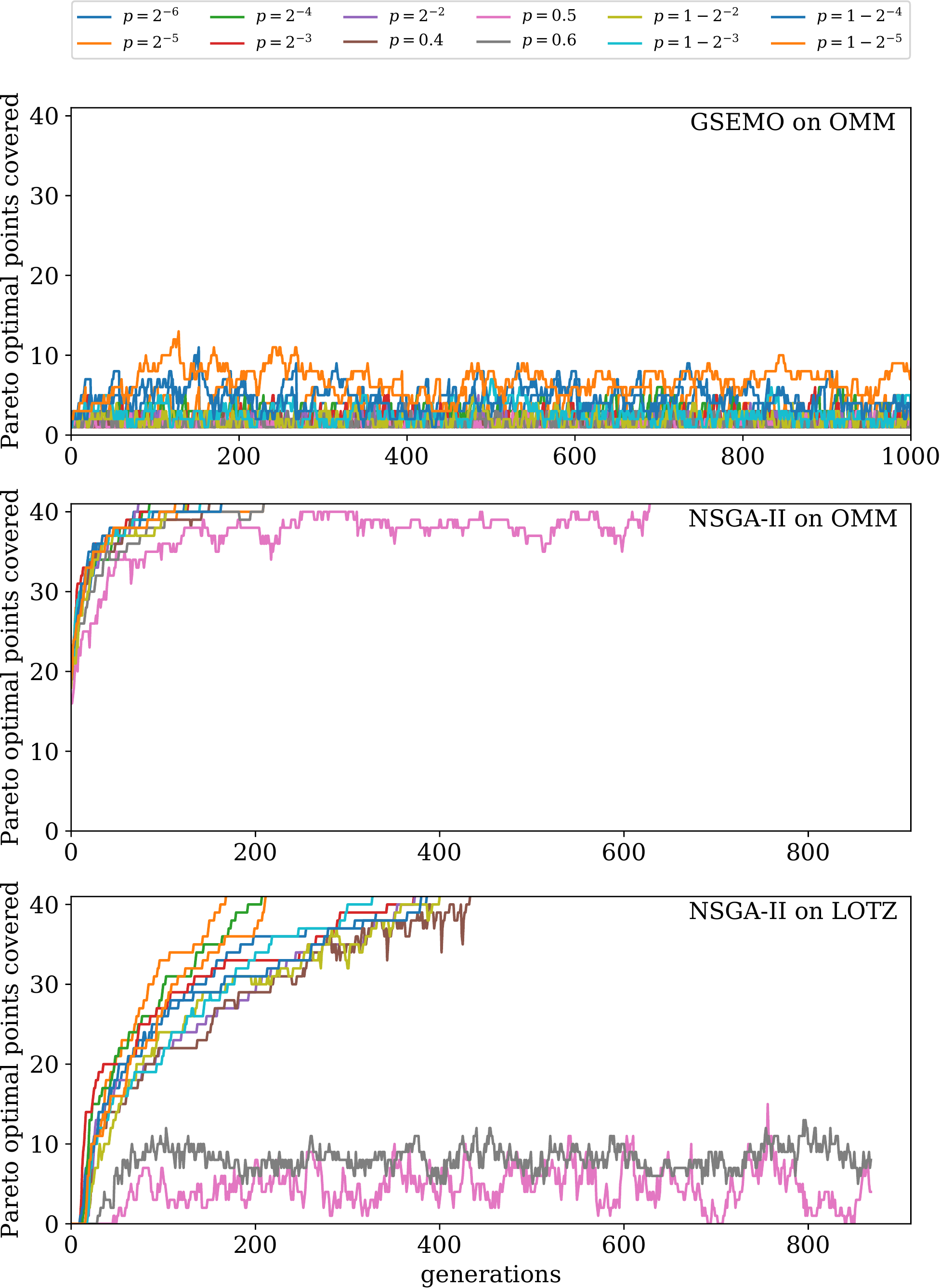}
\caption{Number of Pareto optimal points 
         covered per generation
         by GSEMO and \nsga on \LOTZ, \OMM with $n=40$ under
         the Bernoulli-$(n+1,p)$ noise model in single runs.}
\ifarxiv\else
\Description{Figure: Number of Pareto optimal points covered per generation
             by GSEMO and \nsga on \LOTZ, \OMM with $n=40$.}
\fi
\label{fig:pareto-opt-per-gen}
\end{figure}

As predicted by our results from Section~\ref{sec:analysis-gsemo}, GSEMO failed to cover the Pareto front of both functions within the time limit in all experiments. The first plot of Figure~\ref{fig:pareto-opt-per-gen}
shows the number of different points
on the Pareto front of \OMM being
covered by GSEMO over time for single run on the Bernoulli noise
model. This number never exceeds 16, which is below $40\%$ of the size of the Pareto front.
This is for the easier function \OMM; on \LOTZ the maximum value was~4, that is, less than 10\% of the front size (plot is omitted for lack of space).
The same issue is also evident in the last plot of the figure
for \nsga on \LOTZ with the noise probabilities $p=0.5$ and $p=0.6$.
The success rate (fraction of runs covering the Pareto front) was always 100\% for \nsga in all settings, except for $p \in \{0.5, 0.6\}$ on \LOTZ, where it was 0\%.
This is aligned with our theoretical prediction (Theorem~\ref{thm:nsga-ii-noisy-lotz-p05})
there is a phase transition at $p=1/2$ for \nsga.
%
Note that
in the experiments $p_c = 0.9$ and Theorem~\ref{thm:nsga-ii-noisy-lotz-p05} is for $p_c=0$ and only claims the negative effect for $p<10/19$ which is below $0.6$. The empirical results thus suggest that the findings extend beyond the setting from Theorem~\ref{thm:nsga-ii-noisy-lotz-p05}.

\begin{table}[th]
	\centering
	\begin{tabular}{@{}crrrrrr@{}}
		\toprule
		 & \multicolumn{3}{c}{\LOTZ}  &  \multicolumn{3}{c}{\OMM} \\
        \cmidrule(r){2-4}\cmidrule(r){5-7}
		$p$ & $n=20$ & $n=30$ & $n=40$ & $n=20$ & $n=30$ & $n=40$ \\
        \midrule
		$2^{-6}$ & 24090 & 84051 & 202822 & 12346 & 35982 & 79927\\
		$2^{-5}$ & 23317 & 93687 & 208128 & 10273 & 36706 & 85786\\
		$2^{-4}$ & 24919 & 92294 & 218483 & 12987 & 41440 & 82101\\
		$2^{-3}$ & 27898 & 93297 & 212550 & 12686 & 37987 & 80848\\
		$2^{-2}$ & 31856 & 109701 & 273906 & 14684 & 37235 & 93562\\
        $0.4$ & 34212 & 119540 & 313298 & 16041 & 48848 & 89887 \\
		$0.5$ & \myc 80301 & \myc 270145 & \myc 640453 & 29160 & 95859 & 215240\\
        $0.6$ & \myc 80301 & \myc 270145 & \myc 640453 & 18604 & 50826 & 121581 \\
		$1-2^{-2}$ & 35230 & 156015 & 470869 & 14194 & 51884 & 99495\\
		$1-2^{-3}$ & 29274 & 102933 &221579 & 14948 & 43529 &97910 \\
		$1-2^{-4}$ & 28350 & 89008 &236429 & 12572 & 37402 & 93341\\
		$1-2^{-5}$ & 25692 & 85917 & 216972 & 11705 &39881  &88440 \\
        \bottomrule
	\end{tabular}
	\caption{Average running time of \nsga on \LOTZ and \OMM
             under the $(n+1,p)$-Bernoulli noise model. Runs were stopped after $10n^3$ evaluations. The success rate was always 100\%, except for the shaded cells, where it was~0\%.}
    \label{tab:nsgaii-bernoulli}
    \begin{tabular}{crrrrrr}
		\toprule
		& \multicolumn{3}{c}{\LOTZ}  &  \multicolumn{3}{c}{\OMM} \\
		\cmidrule(r){2-4}\cmidrule(r){5-7}
		$\sigma$ & $n=20$ & $n=30$ & $n=40$ & $n=20$ & $n=30$ & $n=40$ \\
		\midrule
		$n \cdot 2^{-4}$ & 100\% & 100\% & 100\% & 100\% & 100\% & 100\% \\
		$n \cdot 2^{-3}$ & 100\% & 100\% &  96\% & 100\% & 100\% &  70\% \\
		$n \cdot 2^{-2}$ &  65\% &  7\% &  10\% &  40\% &  7\% &  4\% \\
		$n \cdot 2^{-1}$ &  0\% &  0\% &  0\% &  0\% &  0\% &  0\% \\
		$n \cdot 2^0$ &  0\% &  0\% &  0\% &  0\% &  0\% &  0\% \\
		\bottomrule
	\end{tabular}
	\caption{Average success rate of \nsga on \LOTZ and \OMM
		under the Gaussian noise model.}
	\label{tab:nsgaii-gaussian_success_percentage}
\end{table}

Table~\ref{tab:nsgaii-bernoulli} shows the average number of fitness evaluations when running \nsga on
\LOTZ, under the Bernoulli noise model.
We see that the average running time increases as $p$ approaches $1/2$ and that it decreases when approaching $p=0$ or $1$. 
With the Gaussian noise model, \nsga is only able to cover the Pareto front of the noisy
functions
when the standard deviation of the noise is small,
\ie $\sigma \in \{ n\cdot 2^{-4}, n\cdot 2^{-3}\}$.
Table~\ref{tab:nsgaii-gaussian_success_percentage} shows the success rate of \nsga under Gaussian noise. While \nsga is effective for small Gaussian noise, it starts to fail when the standard deviation is increased.

\section{Conclusions}

We have given a first example on which \nsga provably outperforms GSEMO and performed a first theoretical runtime analysis of EMO algorithms in stochastic optimisation. While GSEMO is very sensitive to noise, \nsga can cope well with noise, even when the noise strength is so large that we have domination between noisy and noise-free search points. This holds when the population size is large enough to enable useful search points to survive and when the noise probability is less than $1/2$.
However, for noise probabilities slightly larger than $1/2$, even \nsga requires exponential expected runtime,
thus it experiences a phase transition at $p=1/2$.

There are many open questions for future work. What if noise is applied to all objectives independently?
Can theoretical results be shown for other noise models like Gaussian posterior noise? What mechanisms can prevent noise from disrupting \nsga?

\ifreview
\else

\ifarxiv
\section*{Acknowledgments}
This work benefited from discussions at Dagstuhl seminar~22081 ``Theory of Randomized Optimization Heuristics''.
The third author was supported by the Erasmus+ Programme of the European Union.
\else
\begin{acks}
This work benefited from discussions at Dagstuhl seminar~22081 ``Theory of Randomized Optimization Heuristics''.
The third author was supported by the Erasmus+ Programme of the European Union.
\end{acks}
\fi
\bibliographystyle{ACM-Reference-Format}
\bibliography{references} 


\begin{thebibliography}{53}


\ifx \showCODEN    \undefined \def \showCODEN     #1{\unskip}     \fi
\ifx \showDOI      \undefined \def \showDOI       #1{#1}\fi
\ifx \showISBNx    \undefined \def \showISBNx     #1{\unskip}     \fi
\ifx \showISBNxiii \undefined \def \showISBNxiii  #1{\unskip}     \fi
\ifx \showISSN     \undefined \def \showISSN      #1{\unskip}     \fi
\ifx \showLCCN     \undefined \def \showLCCN      #1{\unskip}     \fi
\ifx \shownote     \undefined \def \shownote      #1{#1}          \fi
\ifx \showarticletitle \undefined \def \showarticletitle #1{#1}   \fi
\ifx \showURL      \undefined \def \showURL       {\relax}        \fi
\providecommand\bibfield[2]{#2}
\providecommand\bibinfo[2]{#2}
\providecommand\natexlab[1]{#1}
\providecommand\showeprint[2][]{arXiv:#2}

\bibitem[Akimoto et~al\mbox{.}(2015)]%
        {Akimoto2015}
\bibfield{author}{\bibinfo{person}{Youhei Akimoto}, \bibinfo{person}{Sandra
  Astete-Morales}, {and} \bibinfo{person}{Olivier Teytaud}.}
  \bibinfo{year}{2015}\natexlab{}.
\newblock \showarticletitle{Analysis of Runtime of Optimization Algorithms for
  Noisy Functions over Discrete Codomains}.
\newblock \bibinfo{journal}{\emph{Theoretical Computer Science}}
  \bibinfo{volume}{605} (\bibinfo{year}{2015}), \bibinfo{pages}{42--50}.
\newblock


\bibitem[Badkobeh et~al\mbox{.}(2015)]%
        {Badkobeh2015}
\bibfield{author}{\bibinfo{person}{Golnaz Badkobeh},
  \bibinfo{person}{Per~Kristian Lehre}, {and} \bibinfo{person}{Dirk Sudholt}.}
  \bibinfo{year}{2015}\natexlab{}.
\newblock \showarticletitle{Black-box Complexity of Parallel Search with
  Distributed Populations}. In \bibinfo{booktitle}{\emph{Proceedings of the
  Foundations of Genetic Algorithms (FOGA'15)}}. \bibinfo{publisher}{{ACM}
  Press}, \bibinfo{pages}{3--15}.
\newblock


\bibitem[Ben-Tal et~al\mbox{.}(2009)]%
        {Ben2009}
\bibfield{author}{\bibinfo{person}{Aharon Ben-Tal}, \bibinfo{person}{Laurent
  El~Ghaoui}, {and} \bibinfo{person}{Arkadi Nemirovski}.}
  \bibinfo{year}{2009}\natexlab{}.
\newblock \bibinfo{booktitle}{\emph{Robust Optimization}}.
\newblock \bibinfo{publisher}{Princeton University Press}.
\newblock


\bibitem[Bian and Qian(2022)]%
        {Bian2022PPSN}
\bibfield{author}{\bibinfo{person}{Chao Bian} {and} \bibinfo{person}{Chao
  Qian}.} \bibinfo{year}{2022}\natexlab{}.
\newblock \showarticletitle{Better Running Time of the Non-dominated Sorting
  Genetic Algorithm {II} ({NSGA-II}) by Using Stochastic Tournament Selection}.
  In \bibinfo{booktitle}{\emph{Proceedings of the International Conference on
  Parallel Problem Solving from Nature (PPSN~'22)}}
  \emph{(\bibinfo{series}{LNCS}, Vol.~\bibinfo{volume}{13399})}.
  \bibinfo{publisher}{Springer}, \bibinfo{pages}{428--441}.
\newblock


\bibitem[Bian et~al\mbox{.}(2021)]%
        {Bian00T21}
\bibfield{author}{\bibinfo{person}{Chao Bian}, \bibinfo{person}{Chao Qian},
  \bibinfo{person}{Yang Yu}, {and} \bibinfo{person}{Ke Tang}.}
  \bibinfo{year}{2021}\natexlab{}.
\newblock \showarticletitle{On the Robustness of Median Sampling in Noisy
  Evolutionary Optimization}.
\newblock \bibinfo{journal}{\emph{Science China Information Sciences}}
  \bibinfo{volume}{64}, \bibinfo{number}{5} (\bibinfo{year}{2021}).
\newblock


\bibitem[Birge and Louveaux(2011)]%
        {Birge2011}
\bibfield{author}{\bibinfo{person}{John~R. Birge} {and}
  \bibinfo{person}{Francois Louveaux}.} \bibinfo{year}{2011}\natexlab{}.
\newblock \bibinfo{booktitle}{\emph{Introduction to Stochastic Programming}}.
\newblock \bibinfo{publisher}{Springer}.
\newblock


\bibitem[Cormen et~al\mbox{.}(2009)]%
        {Cormen2009}
\bibfield{author}{\bibinfo{person}{Thomas~H. Cormen},
  \bibinfo{person}{Charles~E. Leiserson}, \bibinfo{person}{Ronald~L. Rivest},
  {and} \bibinfo{person}{Clifford Stein}.} \bibinfo{year}{2009}\natexlab{}.
\newblock \bibinfo{booktitle}{\emph{Introduction to Algorithms}
  (\bibinfo{edition}{3rd} ed.)}.
\newblock \bibinfo{publisher}{The MIT Press}.
\newblock


\bibitem[Corus et~al\mbox{.}(2021)]%
        {CorusLOW21}
\bibfield{author}{\bibinfo{person}{Dogan Corus}, \bibinfo{person}{Andrei
  Lissovoi}, \bibinfo{person}{Pietro~S. Oliveto}, {and}
  \bibinfo{person}{Carsten Witt}.} \bibinfo{year}{2021}\natexlab{}.
\newblock \showarticletitle{On Steady-State Evolutionary Algorithms and
  Selective Pressure: Why Inverse Rank-Based Allocation of Reproductive Trials
  Is Best}.
\newblock \bibinfo{journal}{\emph{{ACM} Transactions on Evolutionary Learning
  and Optimization}} \bibinfo{volume}{1}, \bibinfo{number}{1}
  (\bibinfo{year}{2021}), \bibinfo{pages}{1--38}.
\newblock


\bibitem[Corus and Oliveto(2018)]%
        {Corus2018a}
\bibfield{author}{\bibinfo{person}{Dogan Corus} {and}
  \bibinfo{person}{Pietro~S. Oliveto}.} \bibinfo{year}{2018}\natexlab{}.
\newblock \showarticletitle{Standard Steady State Genetic Algorithms Can
  Hillclimb Faster Than Mutation-Only Evolutionary Algorithms}.
\newblock \bibinfo{journal}{\emph{IEEE Transactions on Evolutionary
  Computation}} \bibinfo{volume}{22}, \bibinfo{number}{5}
  (\bibinfo{year}{2018}), \bibinfo{pages}{720--732}.
\newblock


\bibitem[Dang et~al\mbox{.}(2022)]%
        {DangELQ22}
\bibfield{author}{\bibinfo{person}{Duc{-}Cuong Dang}, \bibinfo{person}{Anton~V.
  Eremeev}, \bibinfo{person}{Per~Kristian Lehre}, {and} \bibinfo{person}{Xiaoyu
  Qin}.} \bibinfo{year}{2022}\natexlab{}.
\newblock \showarticletitle{Fast Non-elitist Evolutionary Algorithms With
  Power-Law Ranking Selection}. In \bibinfo{booktitle}{\emph{Proceedings of the
  Genetic and Evolutionary Computation Conference (GECCO~'22)}}.
  \bibinfo{publisher}{{ACM} Press}, \bibinfo{pages}{1372--1380}.
\newblock


\bibitem[Dang and Lehre(2014)]%
        {Dang2014}
\bibfield{author}{\bibinfo{person}{Duc{-}Cuong Dang} {and}
  \bibinfo{person}{Per~Kristian Lehre}.} \bibinfo{year}{2014}\natexlab{}.
\newblock \showarticletitle{Evolution under Partial Information}. In
  \bibinfo{booktitle}{\emph{Proceedings of the Genetic and Evolutionary
  Computation Conference {GECCO}~'14}}. \bibinfo{publisher}{{ACM} Press},
  \bibinfo{pages}{1359--1366}.
\newblock


\bibitem[Dang and Lehre(2015)]%
        {Dang2015}
\bibfield{author}{\bibinfo{person}{Duc{-}Cuong Dang} {and}
  \bibinfo{person}{Per~Kristian Lehre}.} \bibinfo{year}{2015}\natexlab{}.
\newblock \showarticletitle{Efficient Optimisation of Noisy Fitness Functions
  with Population-based Evolutionary Algorithms}. In
  \bibinfo{booktitle}{\emph{Proceedings of the Foundations of Genetic
  Algorithms ({FOGA}~'15)}}. \bibinfo{publisher}{{ACM} Press},
  \bibinfo{pages}{62--68}.
\newblock


\bibitem[Dang and Lehre(2016)]%
        {Dang2016}
\bibfield{author}{\bibinfo{person}{Duc{-}Cuong Dang} {and}
  \bibinfo{person}{Per~Kristian Lehre}.} \bibinfo{year}{2016}\natexlab{}.
\newblock \showarticletitle{Runtime Analysis of Non-elitist Populations: From
  Classical Optimisation to Partial Information}.
\newblock \bibinfo{journal}{\emph{Algorithmica}} \bibinfo{volume}{75},
  \bibinfo{number}{3} (\bibinfo{year}{2016}), \bibinfo{pages}{428--461}.
\newblock


\bibitem[Dang et~al\mbox{.}(2017)]%
        {Dang2017}
\bibfield{author}{\bibinfo{person}{Duc-Cuong Dang}, \bibinfo{person}{Tobias
  Friedrich}, \bibinfo{person}{Timo K{\"o}tzing}, \bibinfo{person}{Martin~S.
  Krejca}, \bibinfo{person}{Per~Kristian Lehre}, \bibinfo{person}{Pietro~S.
  Oliveto}, \bibinfo{person}{Dirk Sudholt}, {and} \bibinfo{person}{Andrew~M.
  Sutton}.} \bibinfo{year}{2017}\natexlab{}.
\newblock \showarticletitle{Escaping Local Optima Using Crossover with Emergent
  Diversity}.
\newblock \bibinfo{journal}{\emph{IEEE Transactions on Evolutionary
  Computation}}  \bibinfo{volume}{22} (\bibinfo{year}{2017}),
  \bibinfo{pages}{484--497}.
\newblock
Issue 3.


\bibitem[Dang et~al\mbox{.}(2023)]%
        {Dang2023}
\bibfield{author}{\bibinfo{person}{Duc-Cuong Dang}, \bibinfo{person}{Andre
  Opris}, \bibinfo{person}{Bahare Salehi}, {and} \bibinfo{person}{Dirk
  Sudholt}.} \bibinfo{year}{2023}\natexlab{}.
\newblock \showarticletitle{A Proof that Using Crossover Can Guarantee
  Exponential Speed-Ups in Evolutionary Multi-Objective Optimisation}. In
  \bibinfo{booktitle}{\emph{Proceedings of the AAAI Conference on Artificial
  Intelligence, {AAAI}~2023}}. \bibinfo{publisher}{{AAAI} Press},
  \bibinfo{pages}{to appear, preprint available at
  \url{http://arxiv.org/abs/2301.13687}}.
\newblock


\bibitem[Deb(2011)]%
        {NSGAIICode2011}
\bibfield{author}{\bibinfo{person}{Kalyanmoy Deb}.}
  \bibinfo{year}{2011}\natexlab{}.
\newblock \bibinfo{title}{{NSGA-II} Source Code in {C}, version 1.1.6}.
\newblock
  \bibinfo{howpublished}{\url{https://www.egr.msu.edu/~kdeb/codes/nsga2/nsga2-gnuplot-v1.1.6.tar.gz}}.
\newblock
\newblock
\shownote{Accessed: 2022-08-15}.


\bibitem[Deb et~al\mbox{.}(2002)]%
        {Deb2002}
\bibfield{author}{\bibinfo{person}{Kalyanmoy Deb}, \bibinfo{person}{Amrit
  Pratap}, \bibinfo{person}{Sameer Agarwal}, {and} \bibinfo{person}{T.
  Meyarivan}.} \bibinfo{year}{2002}\natexlab{}.
\newblock \showarticletitle{A Fast and Elitist Multiobjective Genetic
  Algorithm: {NSGA-II}}.
\newblock \bibinfo{journal}{\emph{{IEEE} Transactions on Evolutionary
  Computation}} \bibinfo{volume}{6}, \bibinfo{number}{2}
  (\bibinfo{year}{2002}), \bibinfo{pages}{182--197}.
\newblock


\bibitem[Doerr(2020)]%
        {doerr-theory-chapter}
\bibfield{author}{\bibinfo{person}{Benjamin Doerr}.}
  \bibinfo{year}{2020}\natexlab{}.
\newblock \showarticletitle{Probabilistic Tools for the Analysis of Randomized
  Optimization Heuristics}.
\newblock In \bibinfo{booktitle}{\emph{Theory of Evolutionary Computation:
  Recent Developments in Discrete Optimization}},
  \bibfield{editor}{\bibinfo{person}{Benjamin Doerr} {and}
  \bibinfo{person}{Frank Neumann}} (Eds.). \bibinfo{publisher}{Springer},
  \bibinfo{pages}{1--87}.
\newblock


\bibitem[Doerr(2021)]%
        {DoerrNegativeMultiplicativeDrift}
\bibfield{author}{\bibinfo{person}{Benjamin Doerr}.}
  \bibinfo{year}{2021}\natexlab{}.
\newblock \showarticletitle{{Lower Bounds for Non-elitist Evolutionary
  Algorithms via Negative Multiplicative Drift}}.
\newblock \bibinfo{journal}{\emph{Evolutionary Computation}}
  \bibinfo{volume}{29}, \bibinfo{number}{2} (\bibinfo{date}{06}
  \bibinfo{year}{2021}), \bibinfo{pages}{305--329}.
\newblock
\showISSN{1063-6560}


\bibitem[Doerr et~al\mbox{.}(2015)]%
        {Doerr2015}
\bibfield{author}{\bibinfo{person}{Benjamin Doerr}, \bibinfo{person}{Carola
  Doerr}, {and} \bibinfo{person}{Franziska Ebel}.}
  \bibinfo{year}{2015}\natexlab{}.
\newblock \showarticletitle{From Black-Box Complexity to Designing New Genetic
  Algorithms}.
\newblock \bibinfo{journal}{\emph{Theoretical Computer Science}}
  \bibinfo{volume}{567} (\bibinfo{year}{2015}), \bibinfo{pages}{87--104}.
\newblock


\bibitem[Doerr et~al\mbox{.}(2019)]%
        {Doerr2019opl}
\bibfield{author}{\bibinfo{person}{Benjamin Doerr}, \bibinfo{person}{Christian
  Gie{\ss}en}, \bibinfo{person}{Carsten Witt}, {and} \bibinfo{person}{Jing
  Yang}.} \bibinfo{year}{2019}\natexlab{}.
\newblock \showarticletitle{The {(1+$\lambda$)} {Evolutionary} {Algorithm} with
  Self-Adjusting Mutation Rate}.
\newblock \bibinfo{journal}{\emph{Algorithmica}} \bibinfo{volume}{81},
  \bibinfo{number}{2} (\bibinfo{year}{2019}), \bibinfo{pages}{593--631}.
\newblock


\bibitem[Doerr et~al\mbox{.}(2013)]%
        {Doerr2012c}
\bibfield{author}{\bibinfo{person}{Benjamin Doerr}, \bibinfo{person}{Thomas
  Jansen}, \bibinfo{person}{Dirk Sudholt}, \bibinfo{person}{Carola Winzen},
  {and} \bibinfo{person}{Christine Zarges}.} \bibinfo{year}{2013}\natexlab{}.
\newblock \showarticletitle{Mutation Rate Matters Even When Optimizing
  Monotonic Functions}.
\newblock \bibinfo{journal}{\emph{Evolutionary Computation}}
  \bibinfo{volume}{21}, \bibinfo{number}{1} (\bibinfo{year}{2013}),
  \bibinfo{pages}{1--21}.
\newblock


\bibitem[Doerr et~al\mbox{.}(2017)]%
        {Doerr2017-fastGA}
\bibfield{author}{\bibinfo{person}{Benjamin Doerr}, \bibinfo{person}{Huu~Phuoc
  Le}, \bibinfo{person}{R{\'e}gis Makhmara}, {and} \bibinfo{person}{Ta~Duy
  Nguyen}.} \bibinfo{year}{2017}\natexlab{}.
\newblock \showarticletitle{Fast Genetic Algorithms}. In
  \bibinfo{booktitle}{\emph{Proceedings of the Genetic and Evolutionary
  Computation Conference (GECCO~'17)}}. \bibinfo{publisher}{{ACM} Press},
  \bibinfo{pages}{777--784}.
\newblock


\bibitem[Doerr and Qu(2022)]%
        {DoerrQu22}
\bibfield{author}{\bibinfo{person}{Benjamin Doerr} {and}
  \bibinfo{person}{Zhongdi Qu}.} \bibinfo{year}{2022}\natexlab{}.
\newblock \showarticletitle{A First Runtime Analysis of the {NSGA-II} on a
  Multimodal Problem}. In \bibinfo{booktitle}{\emph{Proceedings of the
  International Conference on Parallel Problem Solving from Nature
  ({PPSN}~'22)}} \emph{(\bibinfo{series}{LNCS}, Vol.~\bibinfo{volume}{13399})}.
  \bibinfo{publisher}{Springer}, \bibinfo{pages}{399--412}.
\newblock


\bibitem[Doerr and Qu(2023a)]%
        {DoerrQu2022a}
\bibfield{author}{\bibinfo{person}{Benjamin Doerr} {and}
  \bibinfo{person}{Zhongdi Qu}.} \bibinfo{year}{2023}\natexlab{a}.
\newblock \showarticletitle{From Understanding the Population Dynamics of the
  {NSGA-II} to the First Proven Lower Bounds}. In
  \bibinfo{booktitle}{\emph{Proceedings of the AAAI Conference on Artificial
  Intelligence, {AAAI}~2023}}. \bibinfo{publisher}{{AAAI} Press},
  \bibinfo{pages}{to appear, preprint available at
  \url{https://arxiv.org/abs/2209.13974}}.
\newblock


\bibitem[Doerr and Qu(2023b)]%
        {DoerrQu2022}
\bibfield{author}{\bibinfo{person}{Benjamin Doerr} {and}
  \bibinfo{person}{Zhongdi Qu}.} \bibinfo{year}{2023}\natexlab{b}.
\newblock \showarticletitle{Runtime Analysis for the {NSGA-II}: Provable
  Speed-Ups From Crossover}. In \bibinfo{booktitle}{\emph{Proceedings of the
  AAAI Conference on Artificial Intelligence, {AAAI}~2023}}.
  \bibinfo{publisher}{{AAAI} Press}, \bibinfo{pages}{to appear, preprint
  available at \url{https://arxiv.org/abs/2208.08759}}.
\newblock


\bibitem[Droste(2004)]%
        {Droste2004}
\bibfield{author}{\bibinfo{person}{Stefan Droste}.}
  \bibinfo{year}{2004}\natexlab{}.
\newblock \showarticletitle{Analysis of the (1+1)~{EA} for a Noisy {OneMax}}.
  In \bibinfo{booktitle}{\emph{Proceedings of the Genetic and Evolutionary
  Computation Conference ({GECCO}~'04)}}. \bibinfo{publisher}{Springer},
  \bibinfo{pages}{1088--1099}.
\newblock


\bibitem[Forman and Selly(2001)]%
        {Forman2001}
\bibfield{author}{\bibinfo{person}{Ernest~H Forman} {and}
  \bibinfo{person}{Mary~Ann Selly}.} \bibinfo{year}{2001}\natexlab{}.
\newblock \bibinfo{booktitle}{\emph{Decision by Objectives: {H}ow to Convince
  Others That You are Right}}.
\newblock \bibinfo{publisher}{World Scientific Publishing}.
\newblock


\bibitem[Friedrich et~al\mbox{.}(2017)]%
        {FriedrichKKS17}
\bibfield{author}{\bibinfo{person}{Tobias Friedrich}, \bibinfo{person}{Timo
  K{\"{o}}tzing}, \bibinfo{person}{Martin~S. Krejca}, {and}
  \bibinfo{person}{Andrew~M. Sutton}.} \bibinfo{year}{2017}\natexlab{}.
\newblock \showarticletitle{The Compact Genetic Algorithm is Efficient Under
  Extreme Gaussian Noise}.
\newblock \bibinfo{journal}{\emph{{IEEE} Transactions on Evolutionary
  Computation}} \bibinfo{volume}{21}, \bibinfo{number}{3}
  (\bibinfo{year}{2017}), \bibinfo{pages}{477--490}.
\newblock


\bibitem[Giel and Lehre(2010)]%
        {Giel2010}
\bibfield{author}{\bibinfo{person}{Oliver Giel} {and}
  \bibinfo{person}{Per~Kristian Lehre}.} \bibinfo{year}{2010}\natexlab{}.
\newblock \showarticletitle{On the Effect of Populations in Evolutionary
  Multi-Objective Optimisation}.
\newblock \bibinfo{journal}{\emph{Evolutionary Computation}}
  \bibinfo{volume}{18}, \bibinfo{number}{3} (\bibinfo{year}{2010}),
  \bibinfo{pages}{335--356}.
\newblock


\bibitem[Gie{\ss}en and K{\"{o}}tzing(2016)]%
        {GiessenK16}
\bibfield{author}{\bibinfo{person}{Christian Gie{\ss}en} {and}
  \bibinfo{person}{Timo K{\"{o}}tzing}.} \bibinfo{year}{2016}\natexlab{}.
\newblock \showarticletitle{Robustness of Populations in Stochastic
  Environments}.
\newblock \bibinfo{journal}{\emph{Algorithmica}} \bibinfo{volume}{75},
  \bibinfo{number}{3} (\bibinfo{year}{2016}), \bibinfo{pages}{462--489}.
\newblock


\bibitem[Goh and Tan(2007)]%
        {Goh2007}
\bibfield{author}{\bibinfo{person}{Chi~Keong Goh} {and}
  \bibinfo{person}{Kay~Chen Tan}.} \bibinfo{year}{2007}\natexlab{}.
\newblock \showarticletitle{An Investigation on Noisy Environments in
  Evolutionary Multiobjective Optimization}.
\newblock \bibinfo{journal}{\emph{{IEEE} Transactions on Evolutionary
  Computation}} \bibinfo{volume}{11}, \bibinfo{number}{3}
  (\bibinfo{year}{2007}), \bibinfo{pages}{354--381}.
\newblock


\bibitem[Hughes(2001)]%
        {Hughes2001}
\bibfield{author}{\bibinfo{person}{Evan~J. Hughes}.}
  \bibinfo{year}{2001}\natexlab{}.
\newblock \showarticletitle{Evolutionary Multi-objective Ranking with
  Uncertainty and Noise}. In \bibinfo{booktitle}{\emph{Proceedings of the First
  International Conference on Evolutionary Multi-Criterion Optimization
  ({EMO}~2001)}} \emph{(\bibinfo{series}{LNCS}, Vol.~\bibinfo{volume}{1993})}.
  \bibinfo{publisher}{Springer}, \bibinfo{pages}{329--343}.
\newblock


\bibitem[Jansen(2013)]%
        {Jansen2013}
\bibfield{author}{\bibinfo{person}{Thomas Jansen}.}
  \bibinfo{year}{2013}\natexlab{}.
\newblock \bibinfo{booktitle}{\emph{Analyzing Evolutionary Algorithms - The
  Computer Science Perspective}}.
\newblock \bibinfo{publisher}{Springer}.
\newblock
\showISBNx{978-3-642-17338-7}


\bibitem[Jorritsma et~al\mbox{.}(2023)]%
        {Lengler2023}
\bibfield{author}{\bibinfo{person}{Joost Jorritsma}, \bibinfo{person}{Johannes
  Lengler}, {and} \bibinfo{person}{Dirk Sudholt}.}
  \bibinfo{year}{2023}\natexlab{}.
\newblock \showarticletitle{Comma Selection Outperform Plus Selection on OneMax
  with Randomly Planted Optima}. In \bibinfo{booktitle}{\emph{Proceedings of
  the Genetic and Evolutionary Computation Conference (GECCO~'23)}}.
  \bibinfo{publisher}{{ACM} Press}, \bibinfo{pages}{to appear}.
\newblock


\bibitem[Laumanns et~al\mbox{.}(2004)]%
        {Laumanns2004}
\bibfield{author}{\bibinfo{person}{Marco Laumanns}, \bibinfo{person}{Lothar
  Thiele}, {and} \bibinfo{person}{Eckart Zitzler}.}
  \bibinfo{year}{2004}\natexlab{}.
\newblock \showarticletitle{Running Time Analysis of Multiobjective
  Evolutionary Algorithms on Pseudo-Boolean Functions}.
\newblock \bibinfo{journal}{\emph{{IEEE} Transactions on Evolutionary
  Computation}} \bibinfo{volume}{8}, \bibinfo{number}{2}
  (\bibinfo{year}{2004}), \bibinfo{pages}{170--182}.
\newblock


\bibitem[Lehre(2011)]%
        {Lehre2010a}
\bibfield{author}{\bibinfo{person}{Per~Kristian Lehre}.}
  \bibinfo{year}{2011}\natexlab{}.
\newblock \showarticletitle{Negative Drift in Populations}. In
  \bibinfo{booktitle}{\emph{Proceedings of the International Conference on
  Parallel Problem Solving from Nature (PPSN~'10)}}
  \emph{(\bibinfo{series}{LNCS}, Vol.~\bibinfo{volume}{6238})}.
  \bibinfo{publisher}{Springer}, \bibinfo{pages}{244--253}.
\newblock


\bibitem[Lehre and Qin(2021)]%
        {LehreQ21}
\bibfield{author}{\bibinfo{person}{Per~Kristian Lehre} {and}
  \bibinfo{person}{Xiaoyu Qin}.} \bibinfo{year}{2021}\natexlab{}.
\newblock \showarticletitle{More Precise Runtime Analyses of Non-elitist {EA}s
  in Uncertain Environments}. In \bibinfo{booktitle}{\emph{Proceedings of the
  Genetic and Evolutionary Computation Conference ({GECCO}~'21)}}.
  \bibinfo{publisher}{{ACM}}, \bibinfo{pages}{1160--1168}.
\newblock


\bibitem[Lehre and Qin(2022)]%
        {LehreQ22}
\bibfield{author}{\bibinfo{person}{Per~Kristian Lehre} {and}
  \bibinfo{person}{Xiaoyu Qin}.} \bibinfo{year}{2022}\natexlab{}.
\newblock \showarticletitle{Self-Adaptation via Multi-Objectivisation: A
  Theoretical Study}. In \bibinfo{booktitle}{\emph{Proceedings of the Genetic
  and Evolutionary Computation Conference (GECCO~'22)}}.
  \bibinfo{publisher}{{ACM}}, \bibinfo{pages}{1417--1425}.
\newblock


\bibitem[Lengler(2020)]%
        {Lengler2020}
\bibfield{author}{\bibinfo{person}{Johannes Lengler}.}
  \bibinfo{year}{2020}\natexlab{}.
\newblock \showarticletitle{A General Dichotomy of Evolutionary Algorithms on
  Monotone Functions}.
\newblock \bibinfo{journal}{\emph{IEEE Transactions on Evolutionary
  Computation}} \bibinfo{volume}{24}, \bibinfo{number}{6}
  (\bibinfo{year}{2020}), \bibinfo{pages}{995--1009}.
\newblock


\bibitem[Lengler and Steger(2018)]%
        {LenglerS18}
\bibfield{author}{\bibinfo{person}{Johannes Lengler} {and}
  \bibinfo{person}{Angelika Steger}.} \bibinfo{year}{2018}\natexlab{}.
\newblock \showarticletitle{Drift Analysis and Evolutionary Algorithms
  Revisited}.
\newblock \bibinfo{journal}{\emph{Combinatorics, Probability and Computing}}
  \bibinfo{volume}{27}, \bibinfo{number}{4} (\bibinfo{year}{2018}),
  \bibinfo{pages}{643--666}.
\newblock


\bibitem[Llor{\`{a}} and Goldberg(2003)]%
        {Llora2003}
\bibfield{author}{\bibinfo{person}{Xavier Llor{\`{a}}} {and}
  \bibinfo{person}{David~E. Goldberg}.} \bibinfo{year}{2003}\natexlab{}.
\newblock \showarticletitle{Bounding the Effect of Noise in Multiobjective
  Learning Classifier Systems}.
\newblock \bibinfo{journal}{\emph{Evolutionary Computation}}
  \bibinfo{volume}{11}, \bibinfo{number}{3} (\bibinfo{year}{2003}),
  \bibinfo{pages}{278--297}.
\newblock


\bibitem[Qian et~al\mbox{.}(2020)]%
        {Qian_Bian_Feng_2020}
\bibfield{author}{\bibinfo{person}{Chao Qian}, \bibinfo{person}{Chao Bian},
  {and} \bibinfo{person}{Chao Feng}.} \bibinfo{year}{2020}\natexlab{}.
\newblock \showarticletitle{Subset Selection by Pareto Optimization with
  Recombination}. In \bibinfo{booktitle}{\emph{Proceedings of the {AAAI}
  Conference on Artificial Intelligence, {AAAI}~2020}}.
  \bibinfo{publisher}{{AAAI} Press}, \bibinfo{pages}{2408--2415}.
\newblock


\bibitem[Qian et~al\mbox{.}(2021)]%
        {QianBYTY21}
\bibfield{author}{\bibinfo{person}{Chao Qian}, \bibinfo{person}{Chao Bian},
  \bibinfo{person}{Yang Yu}, \bibinfo{person}{Ke Tang}, {and}
  \bibinfo{person}{Xin Yao}.} \bibinfo{year}{2021}\natexlab{}.
\newblock \showarticletitle{Analysis of Noisy Evolutionary Optimization When
  Sampling Fails}.
\newblock \bibinfo{journal}{\emph{Algorithmica}} \bibinfo{volume}{83},
  \bibinfo{number}{4} (\bibinfo{year}{2021}), \bibinfo{pages}{940--975}.
\newblock


\bibitem[Qin and Lehre(2022)]%
        {QinL22}
\bibfield{author}{\bibinfo{person}{Xiaoyu Qin} {and}
  \bibinfo{person}{Per~Kristian Lehre}.} \bibinfo{year}{2022}\natexlab{}.
\newblock \showarticletitle{Self-Adaptation via Multi-objectivisation: An
  Empirical Study}. In \bibinfo{booktitle}{\emph{Proceedings of the
  International Conference on Parallel Problem Solving from Nature
  ({PPSN}~'22)}} \emph{(\bibinfo{series}{LNCS}, Vol.~\bibinfo{volume}{13398})}.
  \bibinfo{publisher}{Springer}, \bibinfo{pages}{308--323}.
\newblock


\bibitem[Ravi et~al\mbox{.}(1993)]%
        {Ravi1993}
\bibfield{author}{\bibinfo{person}{R. Ravi}, \bibinfo{person}{Madhav~V.
  Marathe}, \bibinfo{person}{S.~S. Ravi}, \bibinfo{person}{Daniel~J.
  Rosenkrantz}, {and} \bibinfo{person}{Harry B.~Hunt III}.}
  \bibinfo{year}{1993}\natexlab{}.
\newblock \showarticletitle{Many Birds With One Stone: Multi-Objective
  Approximation Algorithms}. In \bibinfo{booktitle}{\emph{Proceedings of the
  Annual {ACM} Symposium on Theory of Computing ({STOC}~'93)}}.
  \bibinfo{publisher}{{ACM} Press}, \bibinfo{pages}{438--447}.
\newblock


\bibitem[Rowe and Sudholt(2014)]%
        {Rowe2013}
\bibfield{author}{\bibinfo{person}{Jonathan~E. Rowe} {and}
  \bibinfo{person}{Dirk Sudholt}.} \bibinfo{year}{2014}\natexlab{}.
\newblock \showarticletitle{The choice of the offspring population size in the
  (1,$\lambda$) evolutionary algorithm}.
\newblock \bibinfo{journal}{\emph{Theoretical Computer Science}}
  \bibinfo{volume}{545} (\bibinfo{year}{2014}), \bibinfo{pages}{20--38}.
\newblock


\bibitem[Sudholt(2017)]%
        {Sudholt2016}
\bibfield{author}{\bibinfo{person}{Dirk Sudholt}.}
  \bibinfo{year}{2017}\natexlab{}.
\newblock \showarticletitle{How Crossover Speeds Up Building-Block Assembly in
  Genetic Algorithms}.
\newblock \bibinfo{journal}{\emph{Evolutionary Computation}}
  \bibinfo{volume}{25}, \bibinfo{number}{2} (\bibinfo{year}{2017}),
  \bibinfo{pages}{237--274}.
\newblock


\bibitem[Tan et~al\mbox{.}(2005)]%
        {Tan2005}
\bibfield{author}{\bibinfo{person}{Kay~Chen Tan}, \bibinfo{person}{Eik~Fun
  Khor}, {and} \bibinfo{person}{Tong~Heng Lee}.}
  \bibinfo{year}{2005}\natexlab{}.
\newblock \bibinfo{booktitle}{\emph{Multiobjective Evolutionary Algorithms and
  Applications}}.
\newblock \bibinfo{publisher}{Springer}.
\newblock
\showISBNx{978-1-85233-836-7}


\bibitem[Teich(2001)]%
        {Teich2001}
\bibfield{author}{\bibinfo{person}{J{\"{u}}rgen Teich}.}
  \bibinfo{year}{2001}\natexlab{}.
\newblock \showarticletitle{Pareto-Front Exploration with Uncertain
  Objectives}. In \bibinfo{booktitle}{\emph{Proceedings of the First
  International Conference on Evolutionary Multi-Criterion Optimization
  ({EMO}~2001)}} \emph{(\bibinfo{series}{LNCS}, Vol.~\bibinfo{volume}{1993})}.
  \bibinfo{publisher}{Springer}, \bibinfo{pages}{314--328}.
\newblock


\bibitem[Zheng and Doerr(2022a)]%
        {Zheng2022}
\bibfield{author}{\bibinfo{person}{Weijie Zheng} {and}
  \bibinfo{person}{Benjamin Doerr}.} \bibinfo{year}{2022}\natexlab{a}.
\newblock \showarticletitle{Better Approximation Guarantees for the {NSGA-II}
  by Using the Current Crowding Distance}. In
  \bibinfo{booktitle}{\emph{Proceedings of the Genetic and Evolutionary
  Computation Conference (GECCO~'22)}}. \bibinfo{publisher}{{ACM} Press},
  \bibinfo{pages}{611--619}.
\newblock


\bibitem[Zheng and Doerr(2022b)]%
        {ZhengArXiv2022}
\bibfield{author}{\bibinfo{person}{Weijie Zheng} {and}
  \bibinfo{person}{Benjamin Doerr}.} \bibinfo{year}{2022}\natexlab{b}.
\newblock \bibinfo{title}{Runtime Analysis for the NSGA-II: Proving,
  Quantifying, and Explaining the Inefficiency For Many Objectives}.
\newblock
\newblock
\urldef\tempurl%
\url{https://arxiv.org/abs/2211.13084}
\showURL{%
\tempurl}


\bibitem[Zheng et~al\mbox{.}(2022)]%
        {ZhengLuiDoerrAAAI22}
\bibfield{author}{\bibinfo{person}{Weijie Zheng}, \bibinfo{person}{Yufei Liu},
  {and} \bibinfo{person}{Benjamin Doerr}.} \bibinfo{year}{2022}\natexlab{}.
\newblock \showarticletitle{A First Mathematical Runtime Analysis of the
  Non-dominated Sorting Genetic Algorithm {II} {(NSGA-II)}}. In
  \bibinfo{booktitle}{\emph{Proceedings of the {AAAI} Conference on Artificial
  Intelligence, {AAAI}~2022}}. \bibinfo{publisher}{{AAAI} Press},
  \bibinfo{pages}{10408--10416}.
\newblock


\end{thebibliography}


\end{document}